\documentclass{article}

\PassOptionsToPackage{numbers}{natbib}
\usepackage[preprint]{neurips_2026}
\usepackage{tcolorbox}
\usepackage{amsmath}

\usepackage[utf8]{inputenc} 
\usepackage[T1]{fontenc}    
\usepackage{hyperref}       
\usepackage{url}            
\usepackage{booktabs}       
\usepackage{amsmath}        
\usepackage{amsfonts}       
\usepackage{nicefrac}       
\usepackage{microtype}      
\usepackage{xcolor}         
\usepackage{graphicx}       
\usepackage{subcaption}     
\usepackage{float}          
\usepackage{wrapfig}        

\usepackage{multirow}
\usepackage{tcolorbox}
\tcbuselibrary{breakable,skins}
\usepackage{amsmath,amssymb,amsthm}
\tcbuselibrary{theorems,breakable,skins}

\newtheorem{remark}{Remark}

\title{Improving Sample Diversity in Autoregressive Text-to-Image Generation via Cluster Truncation}

\author{%
  Trang Nguyen$^{1,2}$, Shuang Wu$^2$, Runyan Tan$^2$, Phillip Howard$^2$ \\
  $^1$University of Massachusetts Amherst, $^2$Thoughtworks \\
  \texttt{tramnguyen@umass.edu, \{shuang.wu, runyan.tan, phillip.howard\}@thoughtworks.com}
}

\begin{document}

\maketitle

\begin{abstract}
While diffusion models achieve state-of-the-art image quality for text-to-image (T2I) generation, recent work has demonstrated that they suffer from sample diversity collapse.
In this work, we investigate whether autoregressive (AR) image generation models can push the Pareto frontier between image quality and sample diversity. 
With recent advances in quality and efficiency, AR models have emerged as a viable alternative to diffusion-based image generation. Beyond enabling new use cases such as interleaved image-text generation, their sequential generation process makes them compatible with a wide range of token-based decoding strategies originally developed to improve diversity in text generation. Motivated by the potential of a better diversity-quality tradeoff in the AR paradigm, we present the first systematic study of sample diversity in AR image generation models. We show that two key properties of AR image generation, persistently high token-level entropy and substantial redundancy in visual token spaces, limit the effectiveness of existing token-level decoding methods for diversity enhancement.
We therefore propose \textit{$p$-less cluster}, a new decoding strategy that performs entropy-based truncation sampling at cluster level rather than at token level. 
We evaluate our approach and baseline decoding methods across four autoregressive T2I models and two datasets using a comprehensive suite of metrics spanning image quality, prompt alignment, and diversity. 
Our results show that $p$-less cluster unlocks the greatest diversity across most evaluated autoregressive T2I models and datasets while maintaining image quality and prompt alignment. 

\end{abstract}

\section{Introduction}
\label{sec:intro}


Text-to-image (T2I) generation has witnessed remarkable progress in recent years, with diffusion-based models such as Flux\cite{flux2024}, Stable Diffusion \cite{stabilityai2024sd35}, and Imagen\cite{imagen} producing photorealistic outputs that closely follow complex text prompts.
Despite this progress, diversity collapse has remained a persistent limitation with T2I generation.
Diversity collapse manifests at two levels: \emph{distributional diversity} (the model ignores rare or underrepresented modes of the data distribution) and \emph{sample diversity} (repeated sampling for the same prompt yields little variation in composition, style, or content)~\citep{Gandikota_2026_WACV}.
Critically, these two axes of diversity are distinct and do not always improve together: for example, SDXL-DMD2~\citep{yin2024improved}, a distilled diffusion model, achieves superior distributional diversity as measured by FID~\citep{heusel2017gans} compared to the base model, yet exhibits markedly lower per-prompt sample diversity~\citep{Gandikota_2026_WACV}.
In this work, we focus on the latter: \textit{sample diversity}, i.e., generating a diverse set of images for the same prompt. This capability is particularly important in creative applications, where users are often presented with multiple generated candidates to choose from~\citep{parmar2026scaling, angelova2025integration}.

Recently, autoregressive (AR) image generation~\citep{sun2024llamagen,wang2024emu3,chern2024anole,chen2025janus} has emerged as a compelling alternative paradigm to diffusion-style T2I generation.
Such models have enabled new use cases such as interleaved image and text generation by unifying language and image generation in a single transformer-based architecture.
This has important implications for how sampling is conducted in T2I generation: by formulating image synthesis as sequential token prediction, AR models can leverage decoding strategies developed previously for language generation to promote diversity.

In the realm of text generation, a wide range of sampling approaches have been proposed which are directly applicable to to AR image generation models. Methods such as top-$p$ \citep{holtzmancurious}, top-$k$ \citep{fan2018hierarchical}, min-$p$ \citep{nguyen2024turning}, and $p$-less \cite{tan2026pless} enable greater diversity in language generation without significantly sacrificing output quality by truncating the token probability distribution to a subset of higher-probability tokens, thereby allowing higher temperature values to be applied during sampling. Despite the utility of these sampling strategies for extracting high-quality diverse outputs from LLMs, prior studies on autoregressive T2I have not studied how choice of truncation-based sampling methods impact the sample diversity of generated images. 

\begin{figure}[t]
    \centering
    \includegraphics[width=\linewidth]{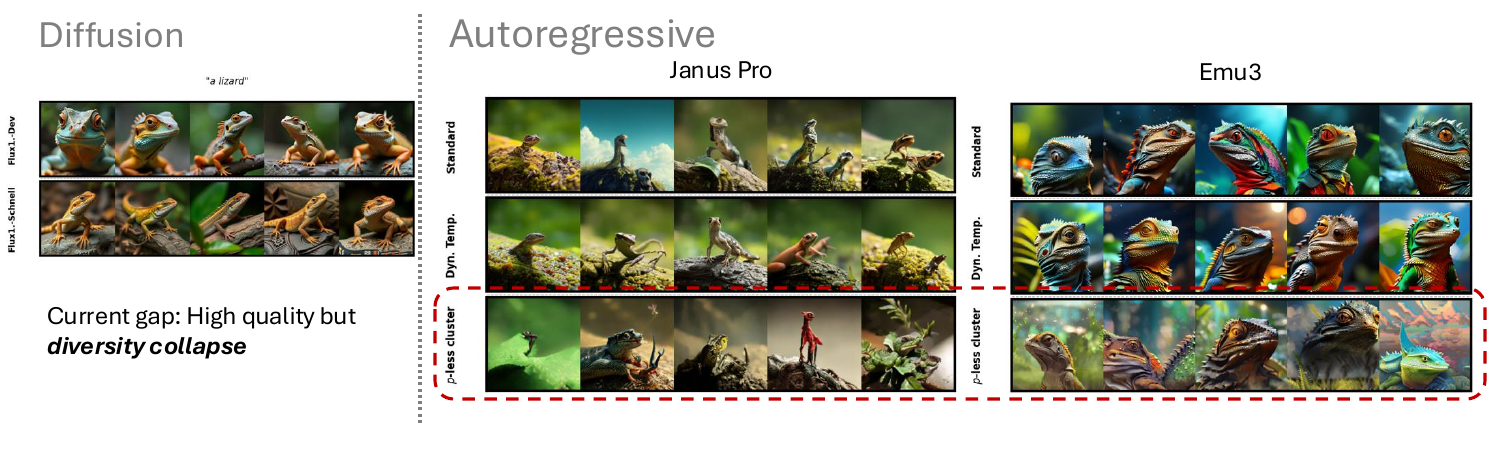}
    \caption{\textit{p-less cluster} enables \textbf{diverse i.i.d.\ sampling} across autoregressive (AR) image generative models while maintaining acceptable image quality, avoiding the issue of diversity collapse.}
    \label{fig:diversity-collapse}
\end{figure}


Our work aims to close this gap by addressing the question: \emph{Can decoding strategies developed for text generation unlock greater sample diversity in autoregressive image models, without significantly sacrificing image quality or prompt alignment?} 
Motivated by the high token-level entropy yet substantial redundancy of visual token spaces in AR image models, we propose \textit{$p$-less cluster}, a cluster-level truncation-based sampling method designed to improve sample diversity without significantly degrading image quality or prompt alignment. Instead of truncating over individual tokens, $p$-less cluster truncates over clusters of similar tokens, encouraging exploration across semantically distinct modes while filtering out low-probability regions that often lead to artifacts.

We conduct comprehensive experiments using four leading AR T2I generation models and benchmark against two diffusion-style T2I model over two prompt datasets. Through a suite of evaluation metrics spanning image quality, prompt alignment, and diversity, we demonstrate how \textit{p-less cluster} extracts greater image diversity across most models and datasets than baseline approaches for sampling. Additionally, we show how our method pushes the boundary of the pareto efficiency frontier characterizing the trade-off between image diversity and quality \cite{parmar2026scaling}. We further substantiate our results with human evaluations, qualitative analyses, and ablation studies on our approach. 

In summary, our contributions are as follows:

\begin{itemize}
    \item We conduct the first study of per-prompt sample diversity in token-wise autoregressive text-to-image models, spanning four models and three different vector-quantized codebook paradigms.
    
    \item We propose \textit{$p$-less cluster}, a cluster-aware decoding strategy that performs two-stage sampling with truncation over clusters of visual tokens.
    
    \item We conduct extensive experiments across four AR image generation models, showing that our method consistently improves the diversity--quality Pareto frontier over standard sampling baselines.
    
    \item Through quantitative and qualitative analyses, we validate the robustness of our findings and identify regimes where cluster-level sampling provides the largest diversity gains.
\end{itemize}

\section{Related work}
\label{sec:related}

\subsection{Diversity collapse in text-to-image generative models}
\label{sec:related:diversity}

A large body of work has focused on improving sample diversity in T2I generation, primarily within diffusion-based models. For example, \citet{parmar2026scaling} propose scaling group inference to generate diverse and high-quality samples, while \citet{wang2025diverse} introduce contrastive noise optimization to encourage diverse outputs during diffusion sampling. 
In contrast, improving diversity in autoregressive (AR) image generation remains underexplored. Existing efforts such as DiverseAR \citep{yang2025diversear} enhance diversity in bitwise or next-scale AR frameworks (e.g., Infinity \cite{han2025infinity}), but do not address next-token AR generation. To address this gap, our work focuses on this setting, aiming to improve diversity under the standard token-by-token autoregressive formulation.

\subsection{Autoregressive image generation}

Pioneering AR image generation operated in pixel space with convolutional  \cite{van2016conditional} or recurrent neural networks \cite{van2016pixel}. Lately, motivated by the success language models, current AR image generation models have adopted transformers and tokenization techniques, which can be broadly categorized into next-token, next-scale, and mask-based paradigms. Next-token prediction generates an image sequentially via discrete tokens, such as via vector quantized Variational Autoencoder \cite{razavi2019generating} or Generative Adversarial Networks \cite{yu2022vector}. \cite{sun2024llamagen} shows that this framework could achieve state-of-the-art performance. Next-scale prediction treats images as a series of multi-resolution maps, performing a coarse-to-fine generation \cite{tian2024visual,han2025infinity,liu2025infinitystar}, with the potential to better model global structural nuances during early stages. Mask-based AR models dispenses with the strict adherence to a predefined sequential generation order \cite{li2024autoregressive,yao2025denoising}. Our work focuses on the next-token paradigm since it is the most mature, benefiting from a vast ecosystem of established techniques in NLP, and serves as a diverse testbed for new methodological improvements \cite{mu2025editar}.

\subsection{Decoding strategies for autoregressive image generation}
\label{sec:related:decoding}

Prior research on decoding strategies for AR image generation has largely focused on improving inference efficiency rather than diversity. For instance, ZipAR \citep{he2024zipar} exploits spatial locality to enable parallel decoding and accelerate generation. Similarly, speculative decoding methods such as Speculative Jacobi Decoding \citep{teng2025speculative} leverage iterative refinement to speed up autoregressive sampling.
Beyond speed, \citet{ma2026towards} propose an entropy-based decoding strategy that adapts token generation order and temperature based on spatial entropy, motivated by the observation that images exhibit non-uniform information density across regions. While this approach provides a principled way to improve both efficiency and quality, it does not explicitly target sample diversity. One common approach to increase I.I.D sample diversity at inference in autoregressive models is to increase temperature applied to discrete token probability distribution. In this work, we focus on decoding strategies designed specifically to enhance diversity at higher temperature in AR image generation, and compare against \citet{ma2026towards} as a strong baseline.


\section{Methodology}
\label{sec:method}
\subsection{Preliminaries}
Let $p_\theta(x \mid c)$ denote an AR image generation model that generates an image $x$ conditioned on a text prompt $c$. Our goal is to design a decoding procedure $\mathcal{D}$ such that a set of $K$ images, sampled independently as
\[
x_1, \ldots, x_K \overset{\text{i.i.d.}}{\sim} \mathcal{D}(p_\theta, c),
\]
achieves higher set-level diversity while maintaining image quality and prompt alignment.
Importantly, we focus on i.i.d.\ sampling without introducing cross-sample coupling or sequential dependence, ensuring that all samples can be generated in parallel and that diversity improvements arise purely from the decoding strategy.

\begin{figure}[t]
  \centering
  \includegraphics[width=\linewidth]{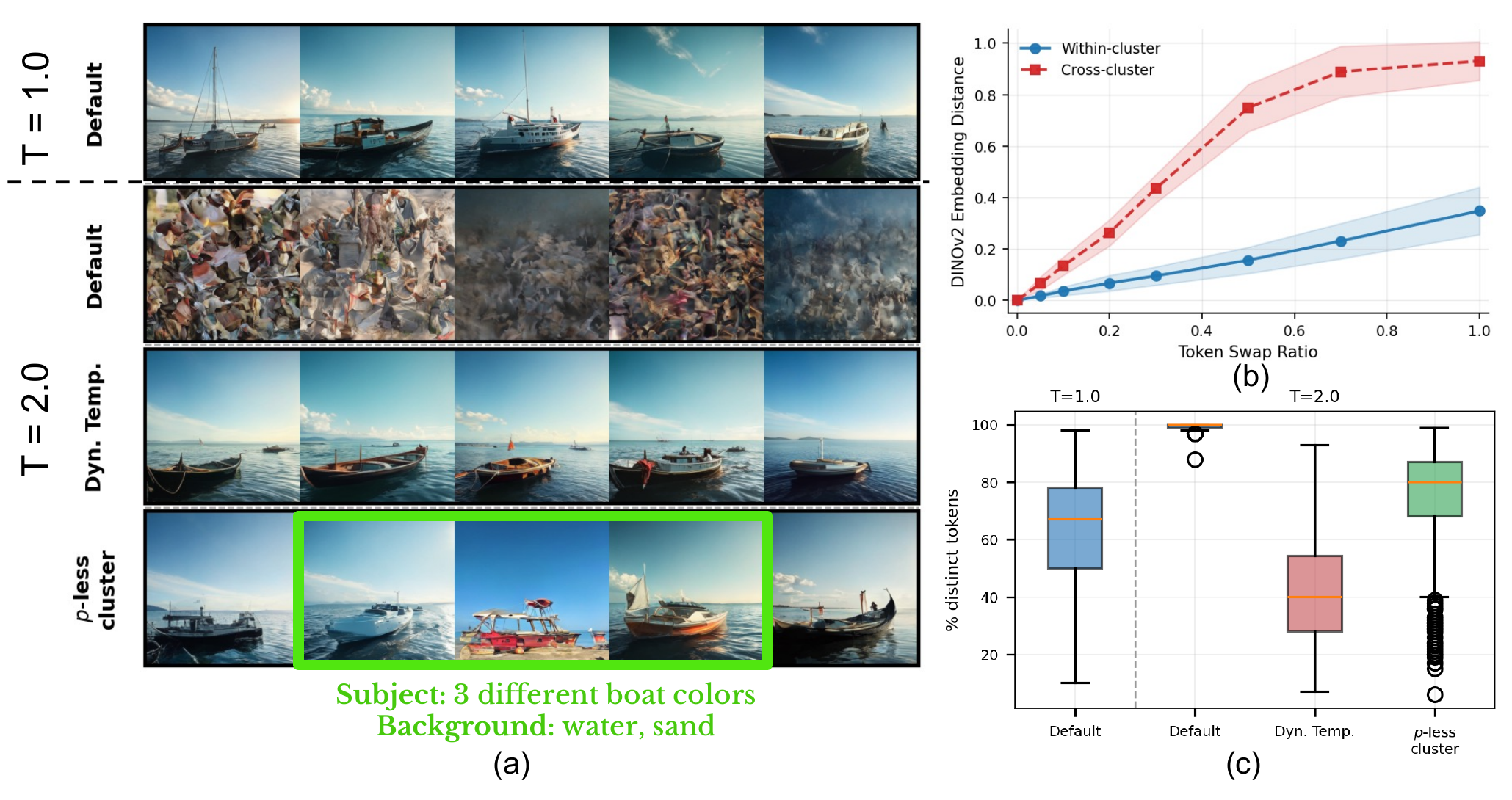}
  \caption{\textbf{Comparison of decoding strategies on image- and token-level diversity (Janus Pro).} (a) Visual comparison; (b) impact of within- vs.\ cross-cluster token swaps on image diversity; (c) percentage of distinct tokens across generation steps (100 samples). At high temperature, default sampling admits too many tokens, leading to degeneration. Dynamic temperature mitigates this by reducing the effective temperature, but limits diversity due to fewer distinct tokens. $p$-less cluster achieves a better tradeoff by combining cluster-level sampling with truncation.}
  \label{fig:motivation}
\end{figure}

\subsection{$p$-less cluster sampling}
We propose a cluster-aware decoding framework for AR image generation that consists of (1) constructing a partition of the token vocabulary, (2) performing cluster-level probability filtering using the $p$-less entropy-based truncation~\cite{pless}, and (3) sampling via a two-stage procedure. Step (1) is executed offline to ensure no substantial cost is incurred at inference time. 

\paragraph{Token Vocabulary Partitioning via K-Means.}
Let $V$ denote the size of the discrete image-token vocabulary. We partition the vocabulary into $N$ clusters $\{\mathcal{C}_n\}_{n=1}^N$ using unsupervised K-Means on codebook vectors (as described in Table~\ref{tab:codebook_analysis}). The resulting partition is computed once offline.

\paragraph{Cluster-Level $p$-less Filtering.}
At each decoding step, the model produces a token distribution $p \in \Delta^{V-1}$, where $\Delta^{V}$ denotes the $V$-dimensional simplex.. We aggregate token probabilities into cluster probabilities:
\[
  q_n = \sum_{v \in \mathcal{C}_n} p_v, \quad n = 1, \dots, N.
\]
We then apply $p$-less~\cite{pless} truncation to the cluster distribution, masking clusters with probability below
\[
  \hat{p}_k = \frac{1}{\exp^{H_k(q)}},
\]
where $H_k$ denotes the Rényi-$k$ entropy, followed by renormalization. This adaptively selects a subset of semantically coherent regions in the vocabulary.

\paragraph{Two-Stage Sampling.}
First, a cluster index $\hat{n}$ is drawn from the filtered cluster distribution. 
Second, a token is sampled from the original token-level probabilities restricted to the selected cluster:
\[
  \hat{v} \sim \frac{p_v}{\sum_{u \in \mathcal{C}_{\hat{n}}} p_u}, \quad v \in \mathcal{C}_{\hat{n}}.
\]
This design allows diversity control at the cluster level while preserving the model’s fine-grained preferences within each cluster.


\subsection{Intuition for cluster-level truncation sampling}
\label{sec:motivation}
Our approach is motivated by two key properties that distinguish visual generation from traditional AR text generation, which are summarized in the following two observations.

\textbf{Observation 1: High token-level entropy during generation.}
Compared to language models, AR image models maintain relatively high entropy across many decoding steps (as shown in Figure \ref{fig:text_vs_image_entropy_hist}).

\textbf{Observation 2: High redundancy in the token space.}
Despite this high entropy, Figure~\ref{fig:motivation}(b) demonstrates that randomly swapping all tokens with other tokens from the same cluster results in relatively small image-level changes, with DINOv2 embedding distance remaining below $0.4$. This indicates that diversity at the discrete token level does not necessarily translate into meaningful diversity at the image level. In contrast, swapping tokens across clusters produces substantially greater image diversity, suggesting that tokens \emph{across} different clusters are more semantically divergent.

Together, these observations explain why default sampling exhibits diversity collapse despite the high entropy of AR image models (Figure~\ref{fig:motivation}(a)). Although many tokens are admitted during sampling, much of this variation occurs within semantically redundant regions of the token space, leading to visually similar outputs.
\citet{ma2026towards} is the first to investigate non-default temperature schedules for AR image generation, proposing a decoding strategy that dynamically decays temperature based on token-level entropy. While entropy-based temperature adjustment has shown strong diversity improvements in text generation~\citep{zhang2403edt}, it is less effective in the visual domain due to the persistently high token-level entropy of AR image models (Observation~1). In practice, high entropy causes the method to reduce the effective temperature during generation, which mitigates artifacts but simultaneously collapses diversity by restricting the effective candidate set (as shown in Figure~\ref{fig:motivation}(c)).

Cluster-level truncation effectively mitigates this limitation.
Leveraging insights from Observation 2, our method operates over clusters of semantically similar tokens to reduce redundancy in the candidate space. By first sampling at the cluster level, $p$-less cluster enables the model to preserve multiple meaningful modes of the distribution rather than collapsing to a set of similar-looking tokens.
However, sampling at high temperature might admits too many clusters which can lead to degeneration, similar to default sampling (Figure~\ref{fig:motivation}(a)(c)). To address this, we combine cluster-level sampling with the $p$-less entropy-based truncation~\cite{pless}, which filters out undesirable clusters based on cluster-level entropy while retaining a diverse set of plausible modes. 

This design yields a favorable balance between diversity and quality, as illustrated in Figure~\ref{fig:motivation}(a). By explicitly encouraging sampling across clusters while controlling tail mass, our method increases the number of distinct tokens (Figure~\ref{fig:motivation}(c)) and promotes cross-cluster exploration without sacrificing image fidelity. Consequently, it achieves higher image-level diversity compared to token-level methods such as dynamic temperature, which reduces diversity by 
overly suppressing non-dominant tokens 
under high token-level entropy.

\subsection{Key design choices}
\label{sec:key_design_choice}
\paragraph{Offline vs. online clustering.} We adopt offline clustering rather than performing clustering at inference time. While online clustering could in principle adapt to the local token distribution at each decoding step, it introduces substantial computational overhead. As inference is already latency-sensitive in AR generation, performing clustering per step would be time cost prohibitive. 

\paragraph{$p$-less for truncation.} Although our cluster-based framework is compatible with a range of truncation strategies, we adopt $p$-less due to its strong empirical performance in high-temperature text generation~\cite{pless}. Unlike fixed-threshold methods such as top-$k$ or top-$p$, $p$-less adapts the truncation threshold based on entropy, making it well-suited to both token-level and cluster-level distributions. This adaptive behavior is particularly important in high-entropy regimes typical of visual generation. 

\paragraph{Hyperparameters.} Besides temperature, which is common to all decoding strategies, our method introduces two additional hyperparameters: (i) $k$, the order of the Rényi entropy used for cluster-level probability truncation, and (ii) $n$, the number of clusters. To ensure practical usability and fair comparison with baselines, we restrict the search space to a small, fixed set: $k \in \{1.5, 2.0\}$ and $n \in \{500, 1000\}$. We find this range sufficient to achieve competitive performance across models.

\section{Experiment setup}
\label{sec:experiments}

\subsection{Base models.}
We apply our decoding strategy to a diverse set of recent AR text-to-image models:
\textbf{JanusPro-7B}~\citep{chen2025janus},
\textbf{LLaMAGen (Stage 1)}~\citep{sun2024llamagen},
\textbf{Emu3}~\citep{wang2024emu3}, and 
\textbf{Anole}~\citep{chern2024anole}.
These models vary in architecture, tokenizer design, and decoding defaults, providing a robust testbed for evaluating decoding strategies across heterogeneous AR systems.
Table~\ref{tab:codebook_analysis} summarizes the visual codebook statistics for each model.

\subsection{Datasets.}
We use two complementary text-to-image benchmarks for evaluation:
\textbf{PartiPrompts}~\citep{yu2022scaling}, a suite of prompts spanning a wide range of semantic categories and difficulty levels, and
\textbf{GenEval}~\citep{ghosh2023geneval}, which focuses on compositional reasoning with object-centric prompts. We select these datasets to cover both abstract and concrete prompt regimes. 
PartiPrompts emphasizes open-ended and imaginative descriptions (e.g. \textit{a smiling banana wearing a bandana}), while GenEval targets precise compositional alignment (e.g., \textit{a photo of a traffic light next to a person}).

For GenEval, we use all $553$ prompts covering six categories (single-object, two-object, counting, colors, color-attribute binding, and compositional relations).
For PartiPrompts, we use all $535$ prompts from the \textit{Basic}, \textit{Complex}, and \textit{Challenge} categories.
For each prompt, we generate a set of $5$ images using the model-specific default classifier-free guidance (CFG) scale: $3.0$ for Anole, $5.0$ for Emu3, $5.0$ for Janus Pro-7B, and $7.5$ for LLaMAGen (unless stated otherwise).

\subsection{Evaluation metrics.}
We adopt established automatic metrics following \citet{parmar2026scaling} while also supplementing them with a VLM-as-judge protocol we developed. This was done to to address known limitations of feature-based scores \cite{hu2023tifa,cho2023davidsonian} such as CLIPScore~\citep{ramesh2022clipscore}, which is widely used to measure both image fidelity and prompt alignment. Despite its popularity, we empirically observed that CLIPScore poorly discriminates between severely distorted images and those with minor artifacts. To better capture perceptual quality and semantic correctness and to provide an interpretable, human-aligned evaluation, we therefore use a vision-language model (InternVL3-14B~\cite{internvl}) to assign scores on a 1--5 scale to each image. All scores are averaged over three independent runs to reduce evaluation variance. See Appendix~\ref{sec:appendix_vlm_judge} for complete details.

Overall, we evaluate performance along three axes: \textit{image quality}, \textit{prompt alignment}, and \textit{sample diversity}, where improving \textit{sample diversity} is the main focus of our work. Image quality and prompt alignment are primarily used as quality constraints to ensure that diversity gains do not come at the expense of unacceptably poor image quality. Following \cite{parmar2026scaling}, we compute standard errors using bootstrapping with $1000$ resamples. 

\textbf{Image quality.} We measure image quality using VLM-as-judge scores (1--5 scale), where the model evaluates perceptual realism and absence of artifacts. 

\textbf{Prompt alignment.} We use (i) VLM-as-judge scores (1--5 scale), and (ii) CLIPScore~\citep{ramesh2022clipscore}.

\textbf{Sample diversity.} We use two metrics computed over a set of 5 generated images per prompt:
(i) DINOv2 embedding distance~\citep{oquab2023dinov2}, and
(ii) Vendi Score \cite{friedman2022vendi}.

\subsection{Baselines}
We compare our method against the following baselines (see complete hyperparameter settings in Appendix~\ref{sec:appendix_impl}): 
\textbf{(1) Default decoding.}
We adopt each model’s standard decoding strategy: multinomial sampling for JanusPro-7B and LLaMAGen, top-$p$ sampling ($p=0.9$) for Anole, and top-$k$ ($k=2048$) for Emu3.
\textbf{(2) Dynamic temperature decoding.}
We include entropy-based dynamic temperature adjustment~\cite{ma2026towards} as a strong baseline. This method modulates the sampling temperature based on token-level entropy, prioritizing regions with higher spatial information density.
\textbf{(3) Diffusion model baselines.}
We additionally evaluate the state-of-the-art Flux.1-Dev and Flux.1-Schnell models \cite{flux2024} using default settings: 50 steps with $\mathrm{CFG} = 3.5$ for Flux.1-Dev, and 4 steps for Flux.1-Schnell. As our work focuses on AR model diversity, we do not compare against specific diffusion-based diversity enhancement techniques.

\section{Results}
\label{sec:results}
\begin{table*}[ht]
\centering
\caption{\textbf{Diversity metrics on GenEval and PartiPrompts at default CFG}. SEs are computed via bootstrapping with 1000 resamples.}
\label{tab:main_results}
\resizebox{0.75\textwidth}{!}{
\begin{tabular}{p{0.5cm} p{1.4cm} p{3.1cm} r r r r}
\toprule
\multirow{2}{*}{Type} & \multirow{2}{*}{Model} & \multirow{2}{*}{Decoding Strategy} & \multicolumn{2}{c}{GenEval} & \multicolumn{2}{c}{PartiPrompts} \\
\cmidrule(lr){4-5}\cmidrule(lr){6-7}
 & & & DINOv2 & Vendi & DINOv2 & Vendi \\
\midrule
\multirow{2}{*}{\rotatebox{90}{Diffusion}} & Flux.1-Dev & Default & 0.311{\scriptsize $\pm$0.016} & 2.336{\scriptsize $\pm$0.069} & 0.354{\scriptsize $\pm$0.013} & 2.514{\scriptsize $\pm$0.056} \\
 & Flux.1-Schnell & Default & 0.270{\scriptsize $\pm$0.006} & 2.186{\scriptsize $\pm$0.024} & 0.277{\scriptsize $\pm$0.008} & 2.227{\scriptsize $\pm$0.033} \\
\midrule
\multirow{12}{*}{\rotatebox{90}{Autoregressive}} & \multirow{3}{*}{Anole} & Default & 0.548{\scriptsize $\pm$0.006} & 3.494{\scriptsize $\pm$0.027} & 0.535{\scriptsize $\pm$0.005} & 3.401{\scriptsize $\pm$0.024} \\
 & & Dynamic Temperature & 0.505{\scriptsize $\pm$0.007} & 3.247{\scriptsize $\pm$0.032} & 0.531{\scriptsize $\pm$0.008} & 3.353{\scriptsize $\pm$0.037} \\
 & & $p$-less cluster (Ours) & \textbf{0.563}{\scriptsize $\pm$0.006} & \textbf{3.549}{\scriptsize $\pm$0.025} & \textbf{0.542}{\scriptsize $\pm$0.007} & \textbf{3.466}{\scriptsize $\pm$0.030} \\
\cmidrule(lr){2-7}
 & \multirow{3}{*}{Emu3} & Default & \textbf{0.441}{\scriptsize $\pm$0.007} & \textbf{2.988}{\scriptsize $\pm$0.030} & 0.364{\scriptsize $\pm$0.007} & 2.618{\scriptsize $\pm$0.032} \\
 & & Dynamic Temperature & 0.405{\scriptsize $\pm$0.007} & 2.809{\scriptsize $\pm$0.029} & 0.360{\scriptsize $\pm$0.007} & 2.611{\scriptsize $\pm$0.031} \\
 & & $p$-less cluster (Ours) & \textbf{0.441}{\scriptsize $\pm$0.007} & 2.972{\scriptsize $\pm$0.031} & \textbf{0.400}{\scriptsize $\pm$0.007} & \textbf{2.794}{\scriptsize $\pm$0.031} \\
\cmidrule(lr){2-7}
 & \multirow{3}{*}{Janus Pro} & Default & \textbf{0.374}{\scriptsize $\pm$0.006} & \textbf{2.684}{\scriptsize $\pm$0.027} & 0.426{\scriptsize $\pm$0.007} & 2.947{\scriptsize $\pm$0.031} \\
 & & Dynamic Temperature & 0.222{\scriptsize $\pm$0.005} & 1.994{\scriptsize $\pm$0.023} & 0.375{\scriptsize $\pm$0.007} & 2.707{\scriptsize $\pm$0.033} \\
 & & $p$-less cluster (Ours) & 0.328{\scriptsize $\pm$0.006} & 2.480{\scriptsize $\pm$0.026} & \textbf{0.454}{\scriptsize $\pm$0.007} & \textbf{3.080}{\scriptsize $\pm$0.031} \\
\cmidrule(lr){2-7}
 & \multirow{3}{*}{LlamaGen} & Default & 0.526{\scriptsize $\pm$0.005} & 3.424{\scriptsize $\pm$0.026} & 0.474{\scriptsize $\pm$0.006} & 3.192{\scriptsize $\pm$0.027} \\
 & & Dynamic Temperature & 0.541{\scriptsize $\pm$0.006} & 3.491{\scriptsize $\pm$0.026} & 0.471{\scriptsize $\pm$0.006} & 3.164{\scriptsize $\pm$0.030} \\
 & & $p$-less cluster (Ours) & \textbf{0.543}{\scriptsize $\pm$0.005} & \textbf{3.511}{\scriptsize $\pm$0.024} & \textbf{0.479}{\scriptsize $\pm$0.006} & \textbf{3.213}{\scriptsize $\pm$0.026} \\
\bottomrule
\end{tabular}
}
\end{table*}
\begin{table*}[h]
\centering
\caption{\textbf{Diversity metrics for Janus Pro-7B at extremely low CFG (CFG$=$1.0).} Bold denotes best mean. $p$-less cluster achieves substantially higher diversity compared to baseline decoding strategies.}
\label{tab:low_cfg}
\resizebox{0.75\textwidth}{!}{
\begin{tabular}{p{1.8cm} p{3.3cm} r r r r}
\toprule
\multirow{2}{*}{Model} & \multirow{2}{*}{Decoding Strategy} & \multicolumn{2}{c}{GenEval} & \multicolumn{2}{c}{PartiPrompts} \\
\cmidrule(lr){3-4}\cmidrule(lr){5-6}
 & & DINOv2 & Vendi & DINOv2 & Vendi \\
\midrule
\multirow{3}{*}{Janus Pro} & Default & 0.442{\scriptsize $\pm$0.008} & 2.977{\scriptsize $\pm$0.034} & 0.553{\scriptsize $\pm$0.010} & 3.489{\scriptsize $\pm$0.045} \\
 & Dynamic Temperature & 0.426{\scriptsize $\pm$0.016} & 2.912{\scriptsize $\pm$0.069} & 0.476{\scriptsize $\pm$0.013} & 3.142{\scriptsize $\pm$0.059} \\
 & $p$-less cluster (Ours) & \textbf{0.552}{\scriptsize $\pm$0.006} & \textbf{3.506}{\scriptsize $\pm$0.028} & \textbf{0.589}{\scriptsize $\pm$0.010} & \textbf{3.646}{\scriptsize $\pm$0.042} \\
\bottomrule
\end{tabular}
}
\end{table*}

\subsection{Main results}
\label{sec:main_results}
\paragraph{AR models achieve higher diversity than diffusion baselines at default CFG.} As shown in Table~\ref{tab:main_results}, all four AR models substantially outperform diffusion baselines on both diversity metrics across GenEval and PartiPrompts under default settings. Relative improvements in DINOv2 range from 72\% to 106\%. This suggests that AR models possess a stronger inherent capacity for diversity, likely due to their token-wise generation process.
Within the diffusion baselines, we observe that Flux.1-Schnell consistently underperforms Flux.1-Dev on both diversity metrics. This is consistent with prior findings that distilled diffusion models trade off diversity for efficiency~\citep{Gandikota_2026_WACV}.

\paragraph{$p$-less cluster sampling improves diversity across AR models while maintaining quality and alignment.}
The results in Table~\ref{tab:main_results} (and Tables \ref{tab:geneval}, \ref{tab:parti}) illustrate how $p$-less cluster sampling consistently achieves the highest diversity while maintaining acceptable quality and alignment on both GenEval and PartiPrompts for Anole and LLaMAGen. For Emu3 and Janus, the gains are more mixed with clear advantage in PartiPrompts but underperform Default sampling in GenEval dataset. Additional analysis on performance gain across prompt categories in Partiprompts reveals that $p$-less cluster shows more advantage in complex and imagination prompts compared to basic simple ones. Full results including quality and alignment metrics are provided in Appendix~\ref{sec:appendix_full_tables}. Though $p$-less cluster observes a decrease in image quality and prompt alignment compared to Default and Dynamic Temperature, it maintains acceptable quality and alignment scores ($> 3.0$) across models and datasets. 

\paragraph{$p$-less cluster increases the diversity gap at low CFG for Janus Pro.}
We observed that the maximum diversity across different decoding methods is below $0.5$ for DINOv2 at default CFG for Janus and Emu3. Therefore, we further investigated how these two models perform in the low-CFG setting. For Emu3, none of the decoding strategies achieves an average VLM alignment score above $3.0$ across all temperatures in the low-CFG regime, indicating severe quality degradation. Table~\ref{tab:low_cfg} provides results for Janus Pro, which did meet the quality requirement in this low-CFG setting. We observe that $p$-less cluster outperforms the baseline methods by an even larger margin in this setting.

\paragraph{$p$-less cluster sampling is most effective for complex, open-ended prompts.}
\begin{figure}[ht]
  \centering
  \begin{subfigure}[t]{0.48\linewidth}
    \includegraphics[width=\linewidth]{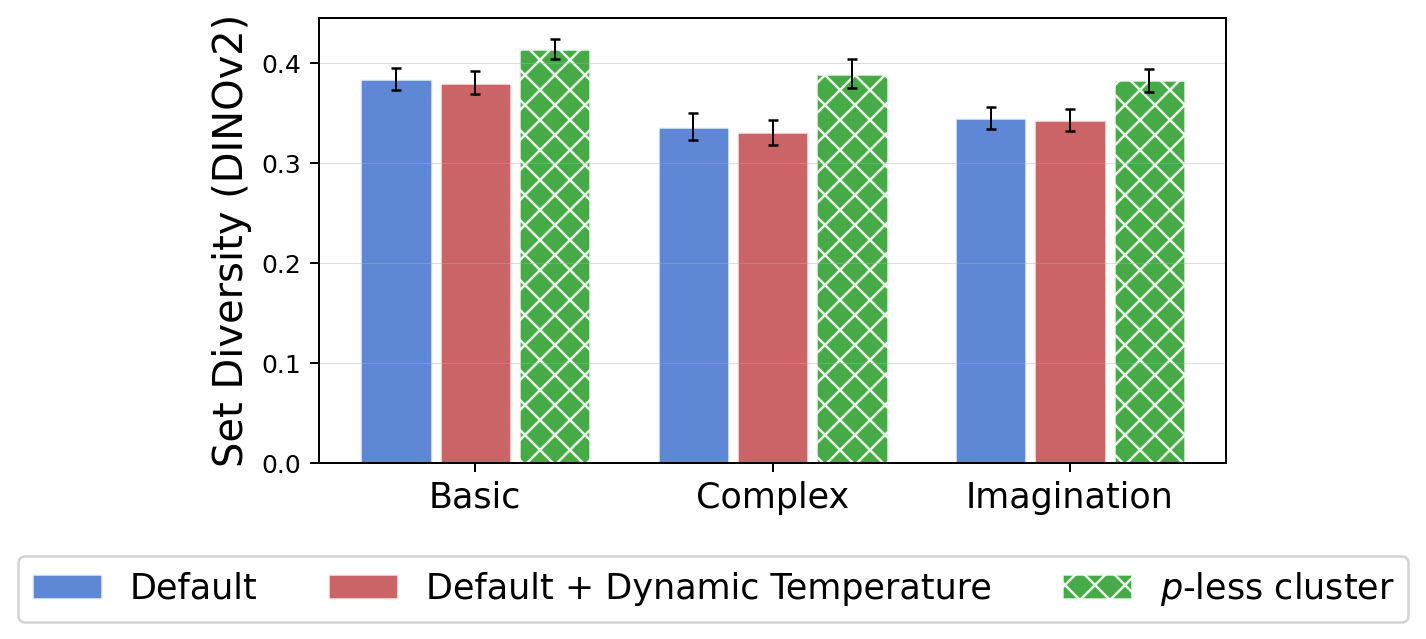}
    \caption{Emu3}
  \end{subfigure}
  \hfill
  \begin{subfigure}[t]{0.48\linewidth}
    \includegraphics[width=\linewidth]{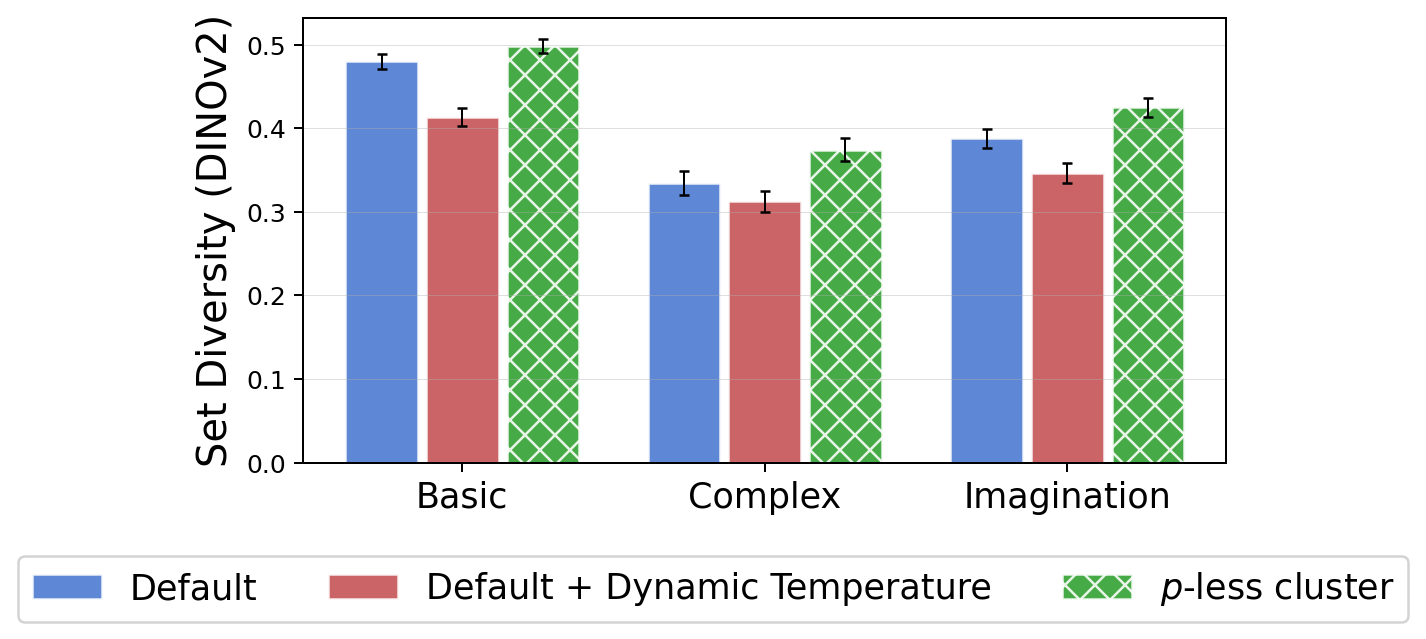}
    \caption{Janus Pro-7B}
  \end{subfigure}
  \caption{\textbf{DINOv2 diversity by PartiPrompts challenge category.} $p$-less cluster achieves larger gains in the \emph{Complex} and \emph{Imagination} categories with more open-ended and abstract prompts.}
  \label{fig:challenge_parti}
\end{figure}
\begin{figure}[h]
  \centering
  \includegraphics[width=0.75\linewidth]{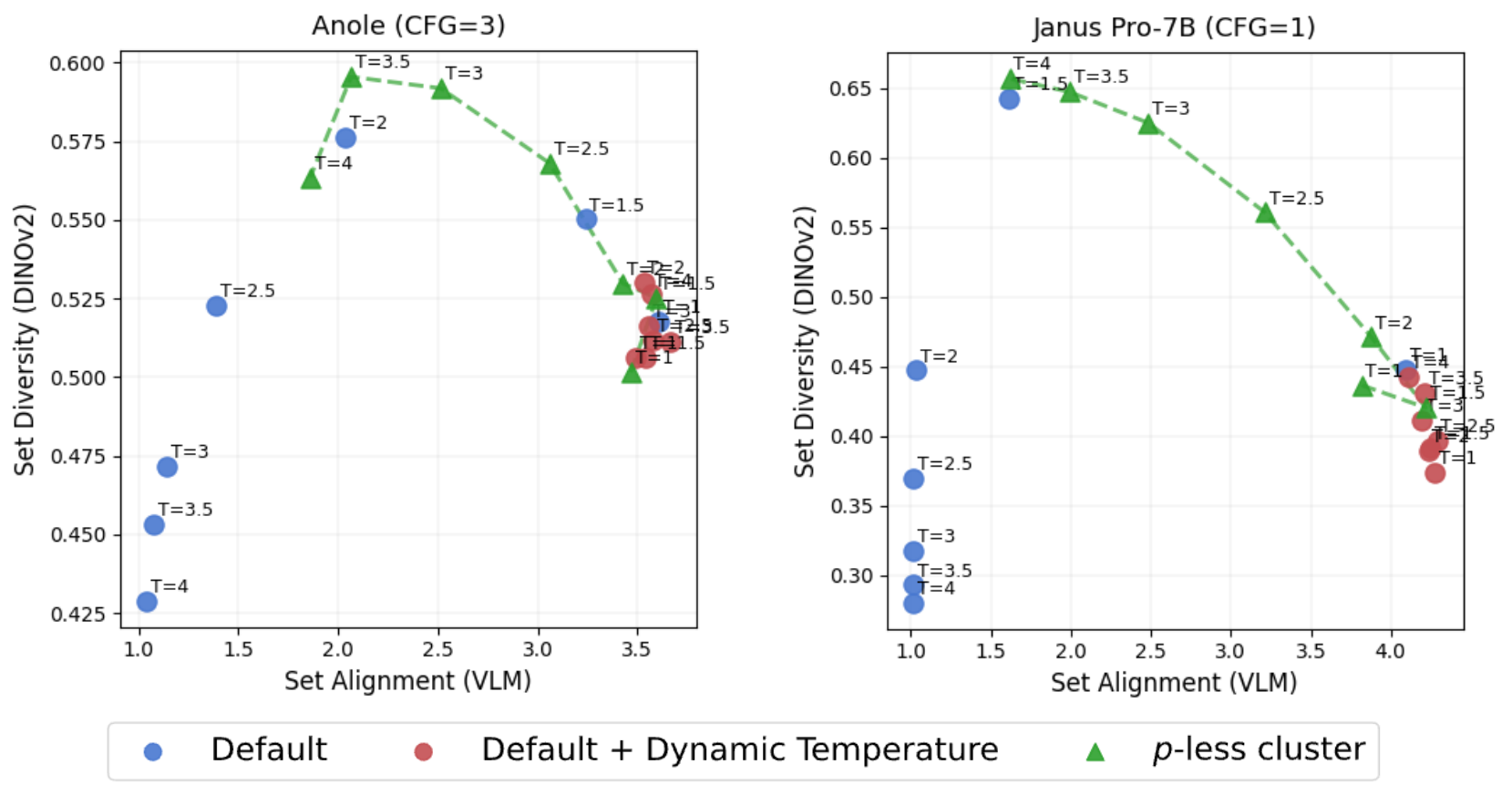}
  \caption{\textbf{Diversity--alignment tradeoff across decoding strategies for Anole and Janus Pro using GenEval Dataset (DINOv2 vs.\ Prompt alignment).} Janus Pro-7B uses low CFG due to its low diversity under default CFG (Section \ref{sec:main_results}). $p$-less cluster shifts the Pareto frontier toward higher diversity while remaining competitive on alignment while Dynamic Temperature maintains same level of high alignment but low diversity across temperatures. Plots for all four models are Figure \ref{fig:apx_diversity_quality_tradeoff}.}
  \label{fig:diversity_quality_tradeoff}
\end{figure}
To better understand these gains, we analyze diversity by prompt category for Emu3 and Janus Pro, the two models where $p$-less cluster shows more mixed results. For both models, diversity improvements are largest in the \textit{Complex} and \textit{Imagination} categories, where prompts admit multiple plausible visual realizations (Figure~\ref{fig:challenge_parti}). In contrast, for more concrete object-centric prompts (e.g., ``a photo of a car''), $p$-less cluster performs similarly to Default sampling for Emu3 and yields only marginal gains for Janus Pro. This trend is less pronounced for Anole and LlamaGen (Figure~\ref{fig:challenge_parti_apx}), suggesting that the effectiveness of $p$-less cluster depends on both prompt complexity and the model’s inherent diversity capacity.

\subsection{Diversity-alignment-quality tradeoff}
\label{sec:diversity-tradeoff}
To characterize the diversity-alignment-quality tradeoff intrinsic to image generation across different decoding strategies, we conducted a full hyperparameter sweep of temperature. Due to the high computational cost of performing full temperature sweep, results were computed using a smaller subset of GenEval ($120$ prompts) and Partiprompts ($90$ prompts) (Appendix~\ref{sec:appendix_impl}). Figure~\ref{fig:diversity_quality_tradeoff} illustrates this diversity-alignment tradeoff frontier for GenEval dataset. Across models and datasets, $p$-less cluster sampling systematically shifts the Pareto frontier toward higher diversity while maintaining comparable alignment, indicating a strictly better tradeoff rather than isolated improvements. In contrast, dynamic temperature adjustment yields inconsistent gains and can reduce diversity in some cases (Table~\ref{tab:main_results}). Furthermore, dynamic temperature adjustment does not offer meaningful tradeoffs between diversity and alignment; instead, different temperature values are clustered closely for all models except LlamaGen (see Figures ~\ref{fig:diversity_quality_tradeoff}) and \ref{fig:apx_diversity_quality_tradeoff}).
Qualitative comparisons in Figure~\ref{fig:motivation} and Figure~\ref{fig:qualitative_emu3} support these findings. Default and dynamic temperature decoding tend to produce visually similar samples for a given prompt, reflecting diversity collapse. In contrast, $p$-less cluster sampling generates a broader range of outputs, with variation in object position, color, and scene composition, while preserving prompt fidelity.

\subsection{Temperature robustness}
\label{sec:temperature-robustness}

We further investigated the robustness of different decoding methods w.r.t. temperature using the same subset of prompts as Section~\ref{sec:diversity-tradeoff}. The results depicted in Figure~\ref{fig:robustness_temperature} of Appendix~\ref{app:temperature-robustness} reveal a fundamental diversity–alignment tradeoff that no decoding strategy fully escapes: higher temperatures consistently boost Set Diversity (DINOv2) but degrade Set Alignment (VLM) across all four models. Beyond a temperature of 3.0, most methods fall below the minimum alignment gate ($\ge 3$). $p$-less cluster sampling offers the best balance. It achieves noticeably higher diversity than Default sampling while staying above the quality gate across a wider temperature range, making it the clear winner for Anole, Emu3, and LlamaGen. Nevertheless, even $p$-less cluster eventually succumbs to alignment degradation at high temperatures.
Dynamic Temperature occupies the opposite end of the spectrum: it maintains the most stable alignment and quality scores and rarely collapses catastrophically, but consistently yields lower peak diversity because it lowers effective temperature when entropy is high. Across all models, its highest-diversity-in-gate points are lower than those of $p$-less cluster.

\subsection{Effect of cluster size}

\begin{figure}[H]
  \centering
  \begin{subfigure}[t]{0.48\textwidth}
    \centering
    \includegraphics[width=\linewidth]{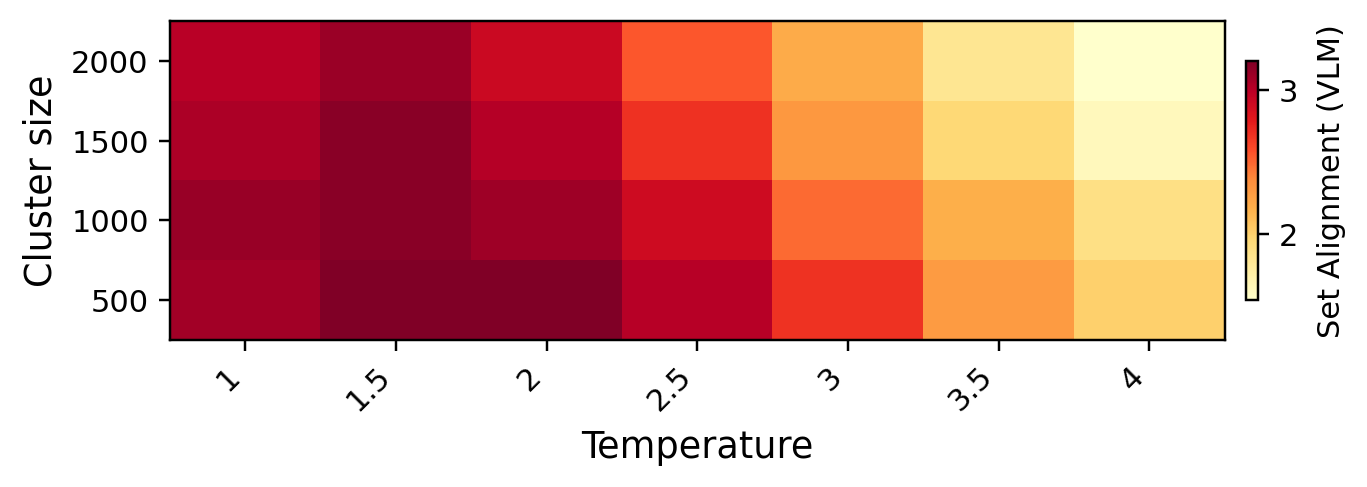}
    \label{fig:cluster_size_alignment}
  \end{subfigure}
  \hfill
  \begin{subfigure}[t]{0.48\textwidth}
    \centering
    \includegraphics[width=\linewidth]{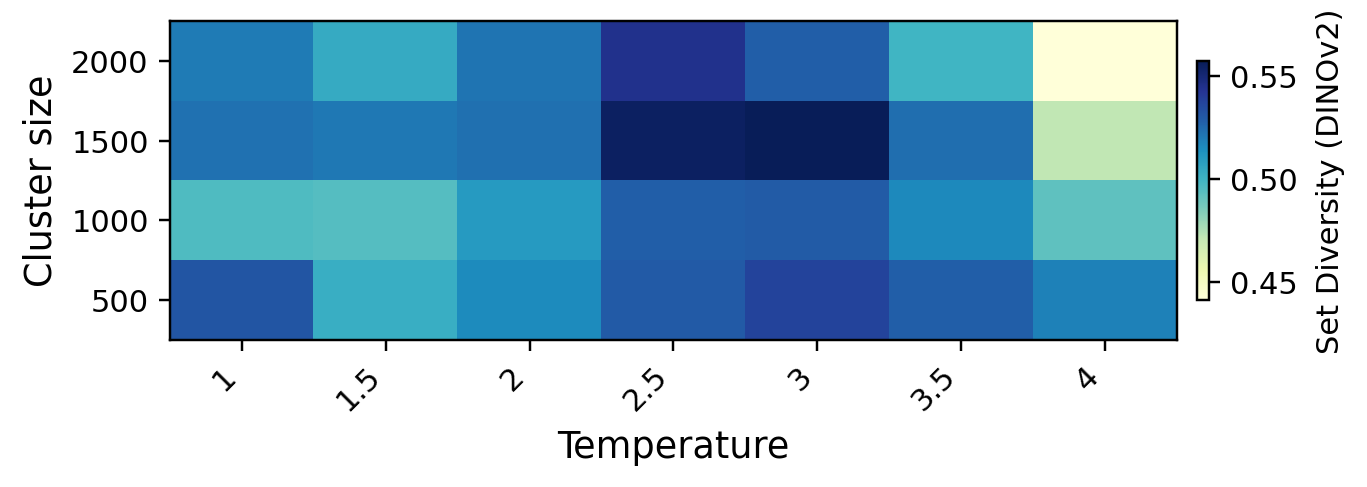}
    \label{fig:cluster_size_diversity}
  \end{subfigure}
  \caption{\textbf{Effect of number of clusters on prompt alignment (left) and diversity (right) for Anole on PartiPrompts.} Smaller number of clusters increases robustness with respect to image quality at higher temperatures and exhibits more stability in sample diversity.}
  \label{fig:effect_of_cluster_size}
\end{figure}

To investigate the effect of cluster size, we fix $k = 2.0$ and conduct experiments on an extended range of cluster size $\{500, 1000, 1500, 2000\}$. As illustrated in Figure \ref{fig:effect_of_cluster_size}, while diversity gains tend to increase with a larger number of clusters, using fewer clusters provides greater robustness to quality degradation across temperatures. This is because a smaller number of clusters induces more aggressive truncation, reducing the variance of the effective candidate set and preventing sampling from low-probability regions that often lead to artifacts. In contrast, a larger number of clusters enables finer-grained control over diversity by preserving more modes of the distribution, but it is also more sensitive to temperature. Retaining low-probability clusters increases diversity, yet it simultaneously raises the risk of sampling from unreliable regions of the distribution, resulting in degraded image quality. For most models, we find $n = 500$ is the optimal number of clusters. More details on model-specific hyperparameter are provided in Appendix \ref{sec:appendix_impl}.

\section{Conclusion}
\label{sec:conclusion}

In this work, we conducted the first comprehensive study of how decoding strategies impact sample diversity in AR T2I models. Inspired by our observation that token-level diversity does not necessarily translate into diversity at the pixel level, we proposed $p$-less cluster: a new truncation-based decoding method for AR T2I which first samples a cluster of visually similar tokens before selecting a token within the cluster. Through extensive experiments and analyses, we demonstrated that $p$-less cluster increases sample diversity across a range of AR T2I models and prompt settings, particularly those which involve more complex or abstract descriptions. Our study highlights the potential for AR models to improve image sample diversity through principled decoding methods which are designed to address the unique characteristics of sampling in AR T2I generation.

\bibliographystyle{abbrvnat}
\bibliography{manual_references}

@article{chen2025janus,
  title={Janus-Pro: Unified Multimodal Understanding and Generation with Data and Model Scaling},
  author={Chen, Xiaokang and Wu, Zhiyu and Liu, Xingchao and Pan, Zizheng and Liu, Wen and Xie, Zhenda and Yu, Xingkai and Ruan, Chong},
  journal={arXiv preprint arXiv:2501.17811},
  year={2025},
  url={https://arxiv.org/abs/2501.17811}
}

@article{sun2024llamagen,
  title={Autoregressive Model Beats Diffusion: Llama for Scalable Image Generation},
  author={Sun, Peize and Jiang, Yi and Chen, Zehuan and Yuan, Jiangmiao and others},
  journal={arXiv preprint arXiv:2406.06525},
  year={2024},
  url={https://arxiv.org/abs/2406.06525}
}

@article{wang2024emu3,
  title={Emu3: Next-token Prediction is All You Need},
  author={Wang, Xinlong and Zhang, Zhengxiong and others},
  journal={arXiv preprint arXiv:2409.18869},
  year={2024},
  url={https://arxiv.org/abs/2409.18869}
}

@article{chern2024anole,
  title={Anole: An Open, Autoregressive, Native Large Multimodal Models for Interleaved Image-Text Generation},
  author={Chern, Ethan and Su, Zhibin and others},
  journal={arXiv preprint arXiv:2407.06135},
  year={2024},
  url={https://arxiv.org/abs/2407.06135}
}

@misc{flux2024,
    author={Black Forest Labs},
    title={FLUX},
    year={2024},
    howpublished={\url{https://github.com/black-forest-labs/flux}},
}

@misc{stabilityai2024sd35,
  author = {Stability AI},
  title = {Stable Diffusion 3.5 Technical Release},
  year = {2024},
  howpublished = {\url{https://stability.ai/news/introducing-stable-diffusion-3-5}},
  note = {Accessed: 2026-05-06}
}

@article{imagen,
  title={Photorealistic text-to-image diffusion models with deep language understanding},
  author={Saharia, Chitwan and Chan, William and Saxena, Saurabh and Li, Lala and Whang, Jay and Denton, Emily L and Ghasemipour, Kamyar and Gontijo Lopes, Raphael and Karagol Ayan, Burcu and Salimans, Tim and others},
  journal={Advances in neural information processing systems},
  volume={35},
  pages={36479--36494},
  year={2022}
}

@inproceedings{yu2022scaling,
  title={Scaling Autoregressive Models for Content-Rich Text-to-Image Generation},
  author={Yu, Jiahui and Xu, Yuanzhong and Koh, Jing Yu and Luong, Thang and Baid, Gunjan and others},
  booktitle={Transactions on Machine Learning Research},
  year={2022},
  url={https://arxiv.org/abs/2206.10789}
}

@inproceedings{ghosh2023geneval,
  title={GenEval: An Object-Focused Framework for Evaluating Text-to-Image Alignment},
  author={Ghosh, Anirban and others},
  booktitle={Advances in Neural Information Processing Systems},
  year={2023},
  url={https://arxiv.org/abs/2310.11513}
}

@article{oquab2023dinov2,
  title={DINOv2: Learning Robust Visual Features without Supervision},
  author={Oquab, Maxime and Darcet, Timoth{\'e}e and others},
  journal={arXiv preprint arXiv:2304.07193},
  year={2023},
  url={https://arxiv.org/abs/2304.07193}
}

@article{friedman2022vendi,
  title={The vendi score: A diversity evaluation metric for machine learning},
  author={Friedman, Dan and Dieng, Adji Bousso},
  journal={arXiv preprint arXiv:2210.02410},
  year={2022}
}

@article{ramesh2022clipscore,
  title={Hierarchical text-conditional image generation with clip latents},
  author={Ramesh, Aditya and Dhariwal, Prafulla and Nichol, Alex and Chu, Casey and Chen, Mark},
  journal={arXiv preprint arXiv:2204.06125},
  volume={1},
  number={2},
  pages={3},
  year={2022}
}

@inproceedings{hu2023tifa,
  title={Tifa: Accurate and interpretable text-to-image faithfulness evaluation with question answering},
  author={Hu, Yushi and Liu, Benlin and Kasai, Jungo and Wang, Yizhong and Ostendorf, Mari and Krishna, Ranjay and Smith, Noah A},
  booktitle={Proceedings of the IEEE/CVF International Conference on Computer Vision},
  pages={20406--20417},
  year={2023}
}

@article{cho2023davidsonian,
  title={Davidsonian scene graph: Improving reliability in fine-grained evaluation for text-to-image generation},
  author={Cho, Jaemin and Hu, Yushi and Garg, Roopal and Anderson, Peter and Krishna, Ranjay and Baldridge, Jason and Bansal, Mohit and Pont-Tuset, Jordi and Wang, Su},
  journal={arXiv preprint arXiv:2310.18235},
  year={2023}
}

@inproceedings{
    parmar2026scaling,
    title={Scaling Group Inference for Diverse and High-Quality Generation},
    author={Gaurav Parmar and Or Patashnik and Daniil Ostashev and Kuan-Chieh Jackson Wang and Kfir Aberman and Srinivasa Narasimhan and Jun-Yan Zhu},
    booktitle={The Fourteenth International Conference on Learning Representations},
    year={2026},
    url={https://openreview.net/forum?id=IyTNxjTuWT}
}

@inproceedings{wang2025diverse,
    title={Diverse Text-to-Image Generation via Contrastive Noise Optimization},
    author={Wang, Fuyun and Teng, Yao and Liu, Xian and others},
    booktitle={Thirty-ninth Conference on Neural Information Processing Systems},
    year={2025},
    url={https://openreview.net/forum?id=EVRMnAREc3}
}

@article{yang2025diversear,
  title={DiverseAR: Boosting Diversity in Bitwise Autoregressive Image Generation},
  author={Yang, Ying and Lv, Zhengyao and Pan, Tianlin and Wang, Haofan and Yang, Binxin and Yin, Hubery and Li, Chen and Si, Chenyang},
  journal={arXiv preprint arXiv:2512.02931},
  year={2025}
}

@inproceedings{han2025infinity,
  title={Infinity: Scaling Bitwise AutoRegressive Modeling for High-Resolution Image Synthesis},
  author={Han, Jian and Liu, Jinlai and Yi, Jiang and Yan, Bin and Zhang, Yuqi and Yuan, Zehuan and Peng, Bingyue and Liu, Xiaobing},
  booktitle={Proceedings of the IEEE/CVF Conference on Computer Vision and Pattern Recognition (CVPR)},
  pages={15733--15744},
  year={2025},
  url={https://openaccess.thecvf.com/content/CVPR2025/html/Han_Infinity_Scaling_Bitwise_AutoRegressive_Modeling_for_High-Resolution_Image_Synthesis_CVPR_2025_paper.html}
}

@article{he2024zipar,
  title={ZipAR: Accelerating Auto-regressive Image Generation through Spatial Locality},
  author={He, Yefei and Chen, Feng and He, Yuanyu and He, Shaoxuan and Zhou, Hong and Zhang, Kaipeng and Zhuang, Bohan},
  journal={arXiv preprint arXiv:2412.04062},
  year={2024}
}

@inproceedings{teng2025speculative,
    title={Speculative Jacobi-Denoising Decoding for Accelerating Autoregressive Text-to-image Generation},
    author={Teng, Yao and Wang, Fuyun and Liu, Xian and others},
    booktitle={The Thirteenth International Conference on Learning Representations},
    year={2025},
    url={https://openreview.net/forum?id=y2eWc6jrlu}
}

@inproceedings{
ma2026towards,
title={Towards Better \& Faster Autoregressive Image Generation: From the Perspective of Entropy},
author={Xiaoxiao Ma and Feng Zhao and Pengyang Ling and Haibo Qiu and Zhixiang Wei and Hu Yu and Jie Huang and Zhixiong Zeng and Lin Ma},
booktitle={The Thirty-ninth Annual Conference on Neural Information Processing Systems},
year={2026},
url={https://openreview.net/forum?id=LiQH1MOCMs}
}

@inproceedings{yin2024improved,
    title={Improved Distribution Matching Distillation for Fast Image Synthesis},
    author={Yin, Tianwei and Gharbi, Micha{\"e}l and Park, Taesung and Zhang, Richard and Shechtman, Eli and Durand, Fredo and Freeman, William T},
    booktitle={NeurIPS},
    year={2024}
}

@inproceedings{heusel2017gans,
  author    = {Heusel, Martin and Ramsauer, Hubert and Unterthiner, Thomas and Nessler, Bernhard and Hochreiter, Sepp},
  booktitle = {Advances in Neural Information Processing Systems},
  pages     = {6626--6637},
  title     = {GANs Trained by a Two Time-Scale Update Rule Converge to a Local Nash Equilibrium},
  volume    = {30},
  year      = {2017}
}

@InProceedings{Gandikota_2026_WACV,
    author    = {Gandikota, Rohit and Bau, David},
    title     = {Distilling Diversity and Control in Diffusion Models},
    booktitle = {Proceedings of the IEEE/CVF Winter Conference on Applications of Computer Vision (WACV)},
    month     = {March},
    year      = {2026},
    pages     = {1304-1313}
}

@article{pless,
  title={p-LESS SAMPLING: A ROBUST HYPERPARAMETER-FREE APPROACH FOR LLM DECODING},
  author={Runyan Tan and Shuang Wu and Phillip Howard},
  journal={arXiv preprint arXiv:2509.23234},
  year={2025},
  note={[cs.AI, cs.CL]},
  url={https://arxiv.org/abs/2509.23234}
}

@article{liu2025infinitystar,
  title={InfinityStar: Unified Spacetime AutoRegressive Modeling for Visual Generation},
  author={Liu, Jinlai and Han, Jian and Yan, Bin and Zhu, Fengda and Wang, Xing and Jiang, Yi and PENG, BINGYUE and Yuan, Zehuan and others},
  journal={Advances in neural information processing systems},
  volume={38},
  year={2025}
}

@article{li2024autoregressive,
  title={Autoregressive image generation without vector quantization},
  author={Li, Tianhong and Tian, Yonglong and Li, He and Deng, Mingyang and He, Kaiming},
  journal={Advances in Neural Information Processing Systems},
  volume={37},
  pages={56424--56445},
  year={2024}
}

@inproceedings{yao2025denoising,
  title={Denoising token prediction in masked autoregressive models},
  author={Yao, Ting and Li, Yehao and Pan, Yingwei and Qiu, Zhaofan and Mei, Tao},
  booktitle={Proceedings of the IEEE/CVF international conference on computer vision},
  pages={18024--18033},
  year={2025}
}

@article{tian2024visual,
  title={Visual autoregressive modeling: Scalable image generation via next-scale prediction},
  author={Tian, Keyu and Jiang, Yi and Yuan, Zehuan and Peng, Bingyue and Wang, Liwei},
  journal={Advances in neural information processing systems},
  volume={37},
  pages={84839--84865},
  year={2024}
}

@inproceedings{mu2025editar,
  title={Editar: Unified conditional generation with autoregressive models},
  author={Mu, Jiteng and Vasconcelos, Nuno and Wang, Xiaolong},
  booktitle={Proceedings of the Computer Vision and Pattern Recognition Conference},
  pages={7899--7909},
  year={2025}
}

@inproceedings{yu2022vector,
  title={Vector-quantized Image Modeling with Improved VQGAN},
  author={Yu, Jiahui and Li, Xin and Koh, Jing Yu and Zhang, Han and Pang, Ruoming and Qin, James and Ku, Alexander and Xu, Yuanzhong and Baldridge, Jason and Wu, Yonghui},
  booktitle={International Conference on Learning Representations},
  year={2022}
}

@article{razavi2019generating,
  title={Generating diverse high-fidelity images with vq-vae-2},
  author={Razavi, Ali and Van den Oord, Aaron and Vinyals, Oriol},
  journal={Advances in neural information processing systems},
  volume={32},
  year={2019}
}

@article{van2016conditional,
  title={Conditional image generation with pixelcnn decoders},
  author={Van den Oord, Aaron and Kalchbrenner, Nal and Espeholt, Lasse and Vinyals, Oriol and Graves, Alex and others},
  journal={Advances in neural information processing systems},
  volume={29},
  year={2016}
}

@inproceedings{van2016pixel,
  title={Pixel recurrent neural networks},
  author={Van Den Oord, A{\"a}ron and Kalchbrenner, Nal and Kavukcuoglu, Koray},
  booktitle={International conference on machine learning},
  pages={1747--1756},
  year={2016},
  organization={PMLR}
}

@article{zhang2403edt,
  title={EDT: Improving Large Language Models’ Generation by Entropy-based Dynamic Temperature Sampling (2024)},
  author={Zhang, Shimao and Bao, Yu and Huang, Shujian},
  journal={URL https://arxiv. org/abs/2403.14541}
}

@article{nguyen2024turning,
  title={Turning up the heat: Min-p sampling for creative and coherent llm outputs},
  author={Nguyen, Minh Nhat and Baker, Andrew and Neo, Clement and Roush, Allen and Kirsch, Andreas and Shwartz-Ziv, Ravid},
  journal={arXiv preprint arXiv:2407.01082},
  year={2024}
}

@article{fan2018hierarchical,
  title={Hierarchical neural story generation},
  author={Fan, Angela and Lewis, Mike and Dauphin, Yann},
  journal={arXiv preprint arXiv:1805.04833},
  year={2018}
}

@inproceedings{holtzmancurious,
  title={The Curious Case of Neural Text Degeneration},
  author={Holtzman, Ari and Buys, Jan and Du, Li and Forbes, Maxwell and Choi, Yejin},
  booktitle={International Conference on Learning Representations}
}

@inproceedings{tan2026pless,
  title={p-less Sampling: A Robust Hyperparameter-Free Approach for LLM Decoding},
  author={Tan, Runyan and Wu, Shuang and Howard, Phillip},
  booktitle={The Fourteenth International Conference on Learning Representations}
}

@article{internvl,
  title={Internvl3: Exploring advanced training and test-time recipes for open-source multimodal models},
  author={Zhu, Jinguo and Wang, Weiyun and Chen, Zhe and Liu, Zhaoyang and Ye, Shenglong and Gu, Lixin and Tian, Hao and Duan, Yuchen and Su, Weijie and Shao, Jie and others},
  journal={arXiv preprint arXiv:2504.10479},
  year={2025}
}

@inproceedings{angelova2025integration,
  title={Integration of artificial intelligence into traditional graphic designer software products--A relief and a challenge},
  author={Angelova, Nadezhda},
  booktitle={AIP Conference Proceedings},
  volume={3274},
  number={1},
  pages={080008},
  year={2025},
  organization={AIP Publishing LLC}
}

\newpage
\appendix

\section{Limitations \& Broader Impacts}
\label{sec:appendix_impact_limit}

\paragraph{Limitations.} Our approach operates purely at inference time, and its effectiveness is therefore bounded by the base model’s underlying token distribution. In particular, gains in diversity depend on the model assigning sufficient probability mass to multiple desirable modes; when the distribution is overly peaked (e.g., as observed for Janus Pro under default CFG), the potential for improvement is inherently limited.

Additionally, the method introduces two hyperparameters (Section~\ref{sec:key_design_choice}). Although we restrict tuning to a small 2x2 combinations, selecting optimal values still requires model-specific calibration.

\paragraph{Broader impacts.} Though our work does not have direct societal impacts, given that AI-generated images are being used for foundation model training, improving sample diversity in generative models mitigates the risk of mode collapse,  ensuring that underrepresented demographic groups and rare cultural nuances are accurately reflected in synthetic data.

\section{Additional analysis of AR image generation}

\subsection{Comparison of entropy distribution between text and image generation}
\label{sec:appendix_entropy_hist}

To characterize how decoding uncertainty differs between language and image generation within the same model architecture, we compare per-token Shannon entropy across text and image generation modes for Janus Pro and Emu3.
For each benchmark prompt, the model is tasked with generating both an image and a descriptive image caption conditioned on the same topic.
Specifically, the text branch receives the instruction: \textit{``Topic: \{prompt\}. Generate a descriptive image caption using $N$ tokens.''}
Both branches are evaluated under identical temperature and CFG settings, ensuring that any observed entropy differences are attributable to the modality rather than generation hyperparameters.

\begin{figure}[ht]
  \centering
  \includegraphics[width=\linewidth]{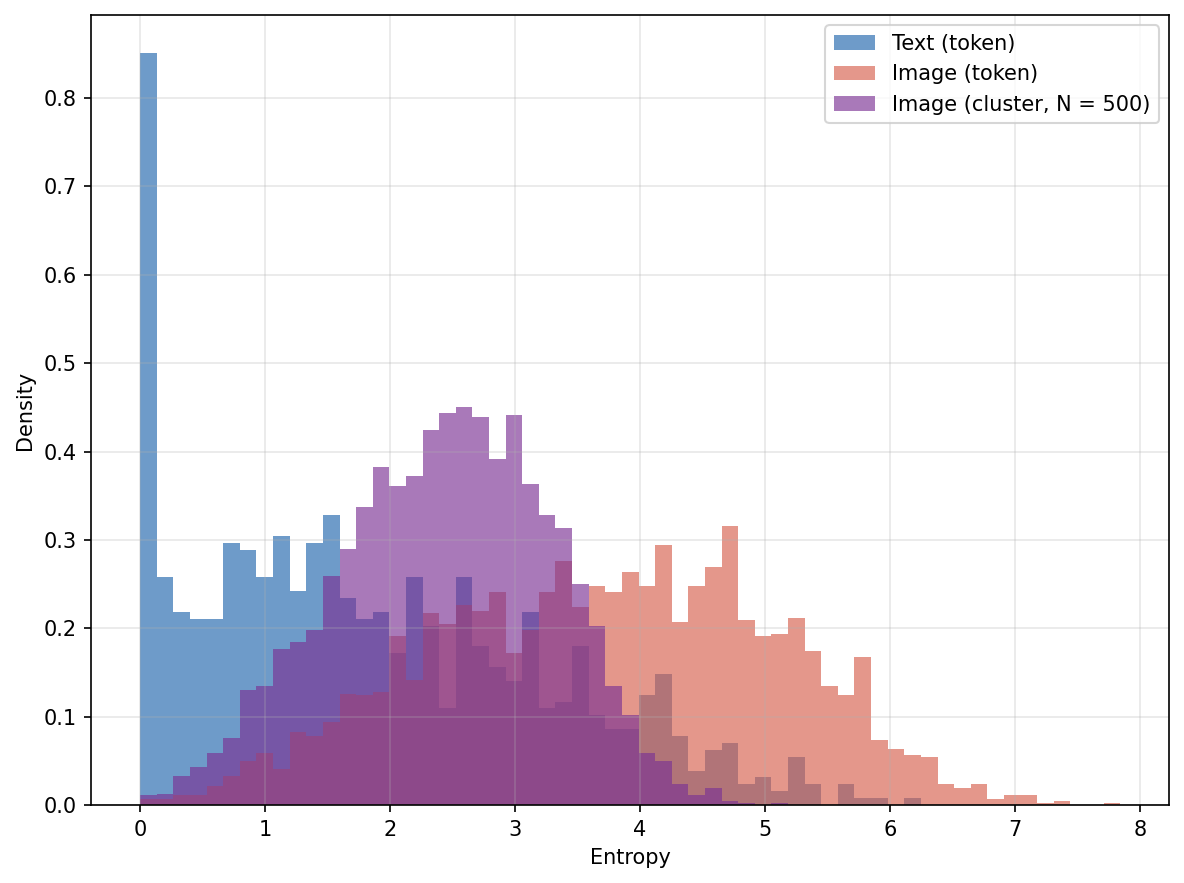}
  \caption{\textbf{Shannon entropy distribution over decoding steps for text vs.\ image generation in Janus Pro, aggregated across 10 prompts.} The image generation mode sustains substantially higher per-token entropy throughout decoding compared to the text generation mode under identical temperature and CFG settings, confirming that AR image models operate in a fundamentally higher-entropy regime.}
  \label{fig:text_vs_image_entropy_hist}
\end{figure}

\subsection{Additional illustration of different decoding strategies at high temperature}
\label{sec:appendix_motivation_extra}
Figure~\ref{fig:apx_motivation} mirrors Figure~\ref{fig:motivation} but adds 2 more methods for comparison: $p$-less without clustering and Default cluster (clustering without $p$-less.

\begin{figure}[t]
  \centering
  \includegraphics[width=\linewidth]{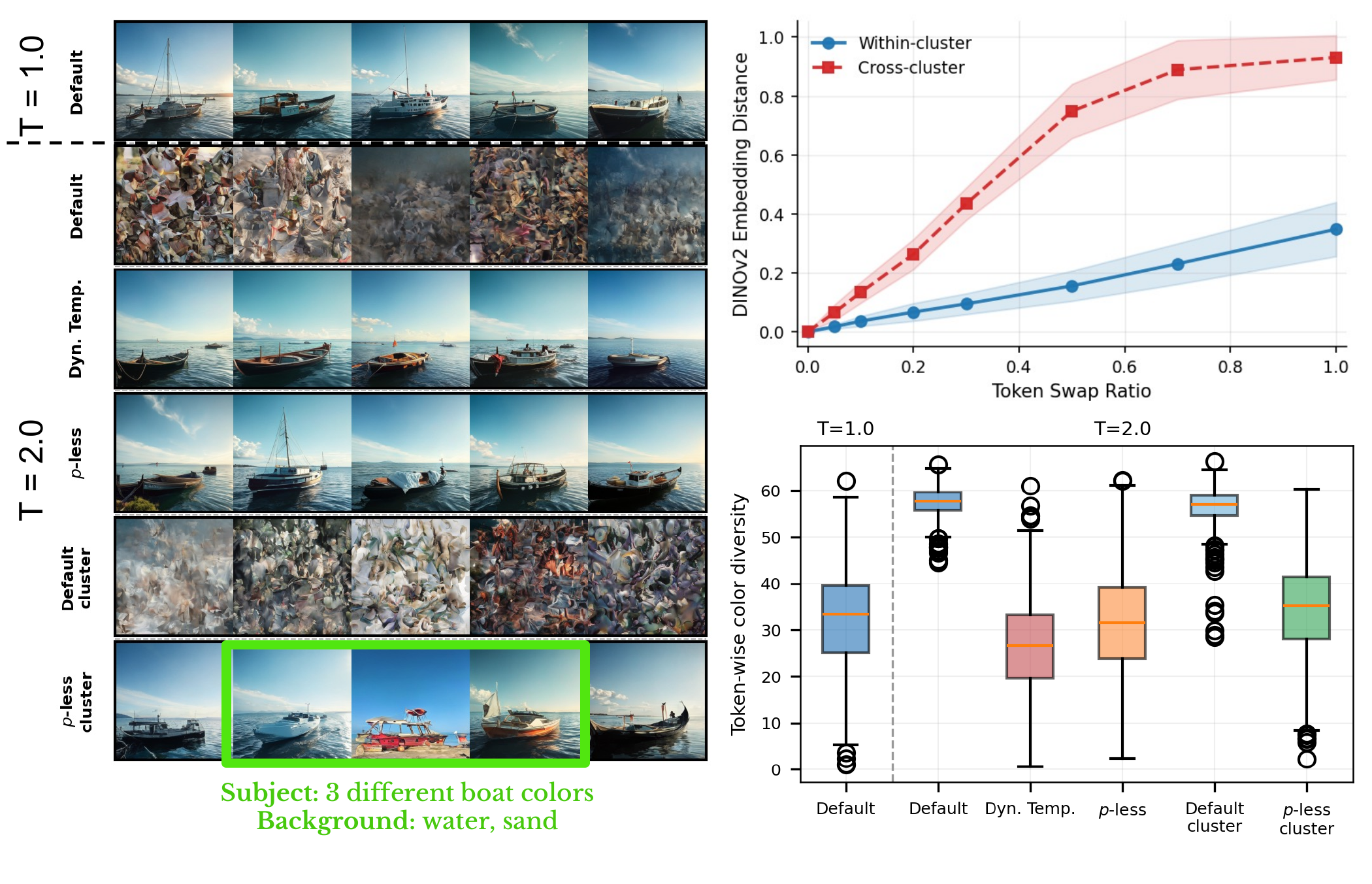}
  \caption{\textbf{Qualitative and quantitative comparison of how different decoding strategies affect image-level and token-level diversity (Janus Pro).} At high temperature, both Default and Default cluster sampling admits too many tokens which leads to degeneration. Both Dynamic Temperature and $p$-less mitigates quality issues at the cost of limited diversity due to low number of distinct tokens. $p$-less cluster strikes a better balance by combining cluster-level sampling (increasing diversity) and truncation (maintaining quality).}
  \label{fig:apx_motivation}
\end{figure}

\subsection{Robustness of decoding strategies w.r.t. temperature}
\label{app:temperature-robustness}

Figure~\ref{fig:robustness_temperature} illustrates the relationship between temperature, set alignment, and set diversity across different models and decoding strategies. See Section~\ref{sec:temperature-robustness} for additional discussion.

\begin{figure}[h]
  \centering  \includegraphics[width=1\linewidth]{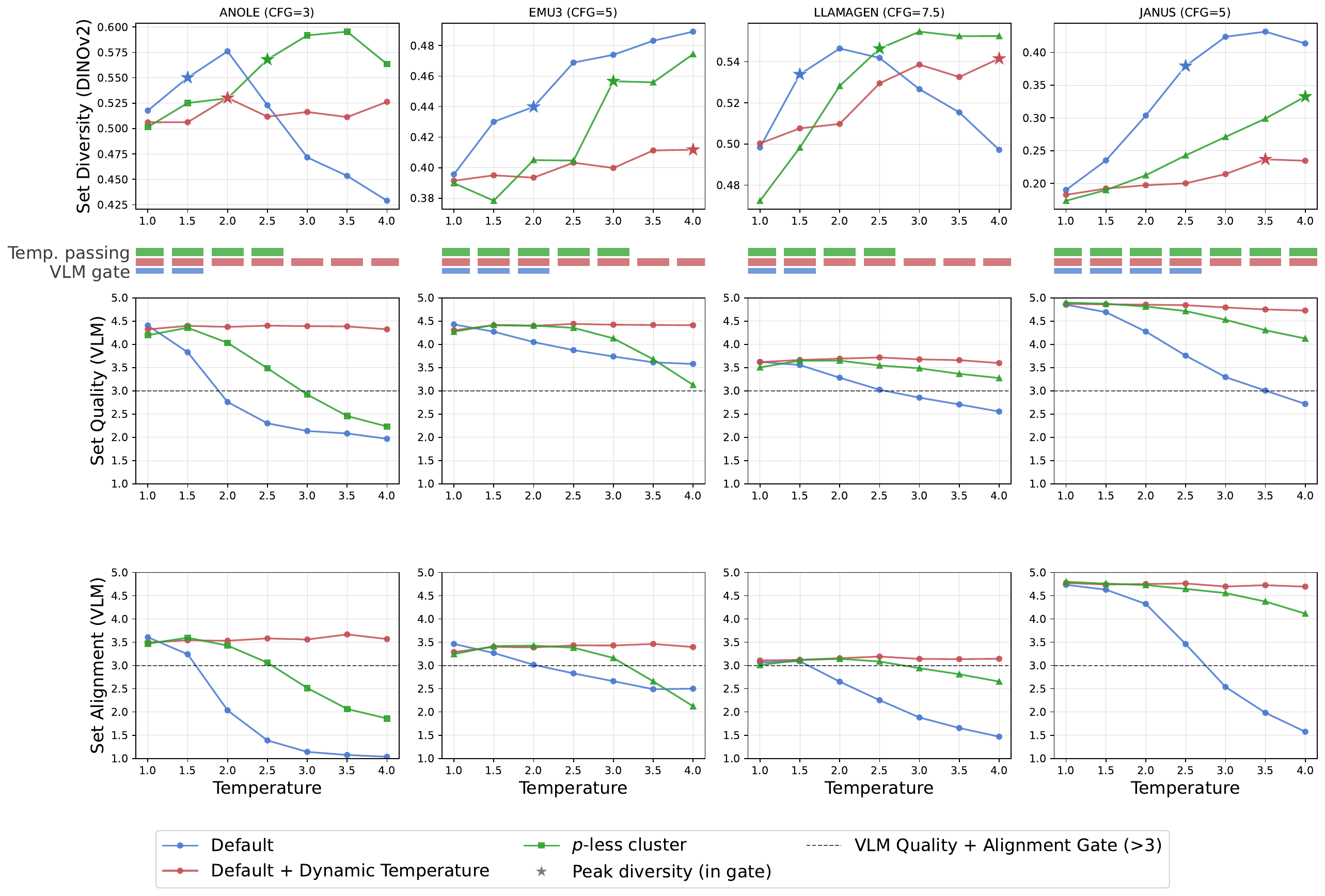}
  \caption{\textbf{Robustness of decoding strategies across sampling temperatures with respect to Diversity (top) and Alignment (bottom) on GenEval.} Starred markers indicate the highest-diversity point that remains above the alignment gate (dashed line at VLM Alignment $> 3.0$). $p$-less cluster (green) remains above the quality gate across a wider range of temperatures than Default sampling, though narrower than Dynamic Temperature. $p$-less cluster consistently achieves higher diversity than Dynamic Temperature across all temperatures and all models.}
  \label{fig:robustness_temperature}
\end{figure}

\begin{figure}[h]
  \centering  \includegraphics[width=1\linewidth]{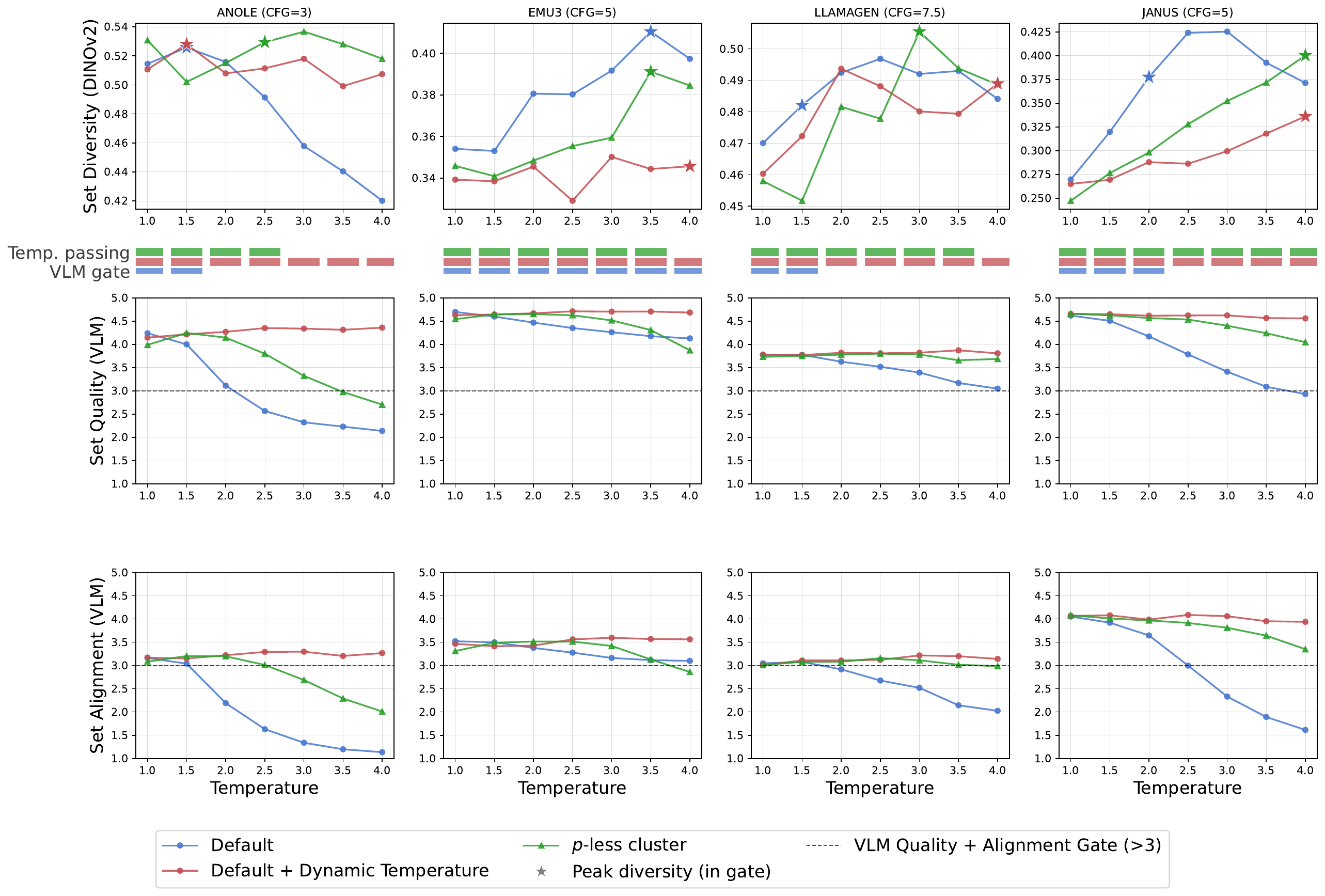}
  \caption{\textbf{Robustness of decoding strategies across sampling temperatures with respect to Diversity (top) and Alignment (bottom) on Partiprompts.} Starred markers indicate the highest-diversity point that remains above the alignment gate (dashed line at VLM Alignment $> 3.0$). $p$-less cluster (green) remains above the quality gate across a wider range of temperatures than Default sampling, though narrower than Dynamic Temperature. $p$-less cluster consistently achieves higher diversity than Dynamic Temperature across all temperatures and all models.}
  \label{fig:robustness_temperature_parti}
\end{figure}

\section{Implementation details}
\label{sec:appendix_impl}
All experiments are conducted on NVIDIA A100 GPUs. We provide implementation details for each model and decoding strategy below.
\subsection{Diffusion models.}\label{diffusion_models}
We use Flux.1-Dev and Flux.1-Schnell~\citep{flux2024} with their default inference settings.\footnote{Default inference settings are from \url{https://github.com/black-forest-labs/flux}.}
For Flux.1-Dev, we apply classifier-free guidance (CFG$=$3.5) with 50 denoising steps.
For Flux.1-Schnell, we use 4 inference steps without CFG.

\subsection{Autoregressive models.}\label{ar_models}
\label{app:ar-model-details}

Additional details of the AR models employed in our study are provided in Table~\ref{tab:codebook_analysis}, including visual codebook statistics and underlying codebook architectures.
\begin{table}[ht]
\centering
\caption{Visual codebook statistics across autoregressive models. Models are grouped by shared codebook architectures.}
\label{tab:codebook_analysis}
\small
\begin{tabular}{@{}llcccccc@{}}
\toprule
\multirow{2}{*}{\textbf{Group}} & \multirow{2}{*}{\textbf{Model}} & \textbf{Vocab} & \textbf{Emb.} & \multicolumn{2}{c}{\textbf{L2 Norm}} & \multicolumn{2}{c}{\textbf{Pairwise Dist.}} \\ \cmidrule(lr){5-6} \cmidrule(lr){7-8} 
 &  & \textbf{Size} & \textbf{Dim.} & \textbf{Mean} & \textbf{Std.} & \textbf{Mean} & \textbf{Median} \\ \midrule
\multirow{1}{*}{VQ-GAN} & \textbf{Anole} & \multirow{1}{*}{8,192} & \multirow{1}{*}{256} & \multirow{1}{*}{7.8043} & \multirow{1}{*}{11.6681} & \multirow{1}{*}{13.9432} & \multirow{1}{*}{13.3072} \\
 \midrule
MoVQ-GAN & \textbf{Emu3} & 32,768 & 4 & 0.4961 & 0.7123 & 0.8541 & 0.8110 \\ \midrule
\multirow{2}{*}{VQ-VAE} & \textbf{Janus Pro 7B} & \multirow{2}{*}{16,384} & \multirow{2}{*}{8} & \multirow{2}{*}{1.0358} & \multirow{2}{*}{0.0325} & \multirow{2}{*}{1.4354} & \multirow{2}{*}{1.4603} \\
 & \textbf{LlamaGen} &  &  &  &  &  &  \\ \bottomrule
\end{tabular}
\end{table}

\textbf{Default decoding.} We use each model’s standard decoding strategy: multinomial sampling for Janus Pro-7B and LLaMAGen, top-$p$ sampling ($p=0.9$) for Anole, and top-$k$ sampling ($k=2048$) for Emu3.

\textbf{Dynamic temperature~\citep{ma2026towards}.} We follow the recommended hyperparameter settings from the original paper for the two hyperparameters: $\alpha$, the decay rate of the temperature, and $\theta$, the lower bound of the temperature. Specifically, we use $\alpha{=}3.0$ and $\theta{=}0.6$ for LLaMAGen, and $\alpha{=}2.5$ and $\theta{=}0.6$ for Anole, Emu3, and Janus Pro-7B.
The latter settings are adopted from those recommended for Lumina-mGPT in the original paper, which the authors note are appropriate for well-trained models that benefit from smoother temperature schedules.

\textbf{$p$-less cluster sampling.} We use a small 2x2 grid for the two hyperparameters: $k$, the moment exponent that controls the sharpness of the cluster distribution, and $n$, the number of clusters. Specifically, we evaluate $k \in \{1.5, 2.0\}$ and $n \in \{500, 1000\}$.

\paragraph{Hyperpameter tuning protocol.}
We conduct hyperparameter tuning for each model, decoding strategy, and benchmark using a smaller subset of prompts: $120$ prompts from GenEval ($20$ prompts randomly sampled from 6 categories) and $90$ prompts from PartiPrompts ($30$ prompts randomly sampled from 3 categories). We select the configuration that maximizes diversity metrics while maintaining acceptable image quality and alignment (VLM Quality and Alignment Score $> 3.0$).

\paragraph{Optimal hyperparameter setting for all methods.}
Table~\ref{tab:hyperparams} summarizes the optimal hyperparameters for each model, decoding strategy, and benchmark. For $p$-less cluster sampling, we find that a smaller number of clusters ($n=500$) generally provides better robustness to quality degradation across temperatures, while a larger number of clusters ($n=1000$) can yield higher diversity gains at the cost of stability. 
\begin{table}[t]
\centering
\caption{%
  Optimal hyperparameters per model, decoding strategy, and benchmark.
  $T$: temperature; $\alpha$, $\theta$: decay rate and lower bound of dynamic
  temperature; $k$: moment exponent; $n$: number of token clusters;
  $p$: nucleus threshold; $K$: top-$K$ vocabulary truncation.
}
\label{tab:hyperparams}
\setlength{\tabcolsep}{7pt}
\renewcommand{\arraystretch}{1.15}
\begin{tabular}{@{}llcc@{}}
\toprule
\textbf{Model} & \textbf{Method} & \textbf{GenEval} & \textbf{PartiPrompts} \\
\midrule
\multirow{3}{*}{Anole}
  & Standard
    & $T{=}1.5,\;p{=}0.9$
    & $T{=}1.5,\;p{=}0.9$ \\
  & Dyn.\ Temp.
    & $T{=}2.0,\;\alpha{=}2.5,\;\theta{=}0.6$
    & $T{=}1.5,\;\alpha{=}2.5,\;\theta{=}0.6$ \\
  & $p$-less Cluster
    & $T{=}2.5,\;k{=}1.5,\;n{=}1000$
    & $T{=}2.5,\;k{=}1.5,\;n{=}500$ \\
\midrule
\multirow{3}{*}{Emu3}
  & Standard
    & $T{=}2.0,\;K{=}4096$
    & $T{=}3.0,\;K{=}4096$ \\
  & Dyn.\ Temp.
    & $T{=}4.0,\;\alpha{=}2.5,\;\theta{=}0.6$
    & $T{=}3.0,\;\alpha{=}2.5,\;\theta{=}0.6$ \\
  & $p$-less Cluster
    & $T{=}3.0,\;k{=}2.0,\;n{=}500$
    & $T{=}3.5,\;k{=}2.0,\;n{=}500$ \\
\midrule
\multirow{3}{*}{Janus Pro-7B}
  & Standard
    & $T{=}2.5$
    & $T{=}2.0$ \\
  & Dyn.\ Temp.
    & $T{=}3.5,\;\alpha{=}2.5,\;\theta{=}0.6$
    & $T{=}4.0,\;\alpha{=}2.5,\;\theta{=}0.6$ \\
  & $p$-less Cluster
    & $T{=}4.0,\;k{=}1.5,\;n{=}500$
    & $T{=}4.0,\;k{=}1.5,\;n{=}500$ \\
\midrule
\multirow{3}{*}{LlamaGen}
  & Standard
    & $T{=}1.5$
    & $T{=}1.5$ \\
  & Dyn.\ Temp.
    & $T{=}4.0,\;\alpha{=}3.0,\;\theta{=}0.6$
    & $T{=}2.0,\;\alpha{=}3.0,\;\theta{=}0.6$ \\
  & $p$-less Cluster
    & $T{=}3.0,\;k{=}1.5,\;n{=}500$
    & $T{=}3.0,\;k{=}2.0,\;n{=}500$ \\
\bottomrule
\end{tabular}
\end{table}

\subsection{Licensing and usage.}
\label{sec:appendix_license}

\begin{table}[H]
\centering
\begin{tabular}{ll}
\toprule
\textbf{Dataset} & \textbf{License} \\
\midrule
GenEval~\citep{ghosh2023geneval} & MIT \\
PartiPrompts~\citep{yu2022scaling} & Apache-2.0 \\
\bottomrule
\end{tabular}
\end{table}

\begin{table}[H]
\centering
\begin{tabular}{ll}
\toprule
\textbf{Model} & \textbf{License} \\
\midrule
Janus Pro~\citep{chen2025janus} & MIT \\
LlamaGen~\citep{sun2024llamagen} & Apache-2.0 \\
Anole~\citep{chern2024anole} & CC BY 4.0 \\
Emu3~\citep{wang2024emu3} & Apache-2.0 \\
FLUX.1-Dev~\citep{flux2024} & FLUX.1-dev Non-Commercial \\
FLUX.1-Schnell~\citep{flux2024} & Apache-2.0 \\
\bottomrule
\end{tabular}
\end{table}

\section{Addtional experimental results}
\subsection{Full quality-alignment-diversity results}
\label{sec:appendix_full_tables}
Tables~\ref{tab:geneval} and~\ref{tab:parti} report the complete results on GenEval and PartiPrompts, including quality, alignment, and diversity metrics for all model–strategy pairs.

\begin{table*}[ht]
\centering
\caption{Full results on the \textbf{GenEval} benchmark comparing Default, Dynamic Temperature, and $p$-less cluster decoding strategies. Quality is measured by a VLM judge score, alignment by VLM and CLIPScore, and diversity by DINOv2 and Vendi Score. Bold indicates the best value per model per metric.}
\label{tab:geneval}
\begin{tabular}{p{1.6cm} p{3.5cm} r r r r r}
\toprule
\multirow{2}{*}{Model} & \multirow{2}{*}{Decoding Strategy} & \multirow{2}{*}{Quality} & \multicolumn{2}{c}{Alignment} & \multicolumn{2}{c}{Diversity} \\
\cmidrule(lr){4-5}\cmidrule(lr){6-7}
 & & VLM & VLM & CLIPScore & DINOv2 & Vendi \\
\midrule
\multirow{3}{*}{Anole} & Default & 3.820{\scriptsize $\pm$0.022} & 3.207{\scriptsize $\pm$0.042} & 0.305{\scriptsize $\pm$0.001} & 0.548{\scriptsize $\pm$0.006} & 3.494{\scriptsize $\pm$0.027} \\
 & Dynamic Temperature & \textbf{4.369}{\scriptsize $\pm$0.020} & \textbf{3.572}{\scriptsize $\pm$0.047} & \textbf{0.316}{\scriptsize $\pm$0.001} & 0.505{\scriptsize $\pm$0.007} & 3.247{\scriptsize $\pm$0.032} \\
 & $p$-less cluster (Ours) & 3.483{\scriptsize $\pm$0.021} & 3.038{\scriptsize $\pm$0.042} & 0.302{\scriptsize $\pm$0.001} & \textbf{0.563}{\scriptsize $\pm$0.006} & \textbf{3.549}{\scriptsize $\pm$0.025} \\
\midrule
\multirow{3}{*}{Emu3} & Default & 4.019{\scriptsize $\pm$0.023} & 2.961{\scriptsize $\pm$0.046} & 0.287{\scriptsize $\pm$0.001} & \textbf{0.441}{\scriptsize $\pm$0.007} & \textbf{2.988}{\scriptsize $\pm$0.030} \\
 & Dynamic Temperature & \textbf{4.375}{\scriptsize $\pm$0.020} & \textbf{3.381}{\scriptsize $\pm$0.047} & \textbf{0.301}{\scriptsize $\pm$0.001} & 0.405{\scriptsize $\pm$0.007} & 2.809{\scriptsize $\pm$0.029} \\
 & $p$-less cluster (Ours) & 4.084{\scriptsize $\pm$0.022} & 3.076{\scriptsize $\pm$0.045} & 0.294{\scriptsize $\pm$0.001} & \textbf{0.441}{\scriptsize $\pm$0.007} & 2.972{\scriptsize $\pm$0.031} \\
\midrule
\multirow{3}{*}{Janus Pro-7B} & Default & 3.760{\scriptsize $\pm$0.026} & 3.475{\scriptsize $\pm$0.038} & 0.303{\scriptsize $\pm$0.001} & \textbf{0.374}{\scriptsize $\pm$0.006} & \textbf{2.684}{\scriptsize $\pm$0.027} \\
 & Dynamic Temperature & \textbf{4.759}{\scriptsize $\pm$0.028} & \textbf{4.762}{\scriptsize $\pm$0.037} & \textbf{0.337}{\scriptsize $\pm$0.001} & 0.222{\scriptsize $\pm$0.005} & 1.994{\scriptsize $\pm$0.023} \\
 & $p$-less cluster (Ours) & 4.097{\scriptsize $\pm$0.022} & 4.059{\scriptsize $\pm$0.033} & 0.320{\scriptsize $\pm$0.001} & 0.328{\scriptsize $\pm$0.006} & 2.480{\scriptsize $\pm$0.026} \\
\midrule
\multirow{3}{*}{LlamaGen} & Default & 3.575{\scriptsize $\pm$0.023} & \textbf{3.101}{\scriptsize $\pm$0.044} & \textbf{0.296}{\scriptsize $\pm$0.001} & 0.526{\scriptsize $\pm$0.005} & 3.424{\scriptsize $\pm$0.026} \\
 & Dynamic Temperature & \textbf{3.610}{\scriptsize $\pm$0.023} & 3.082{\scriptsize $\pm$0.042} & 0.295{\scriptsize $\pm$0.001} & 0.541{\scriptsize $\pm$0.006} & 3.491{\scriptsize $\pm$0.026} \\
 & $p$-less cluster (Ours) & 3.485{\scriptsize $\pm$0.041} & 2.942{\scriptsize $\pm$0.089} & 0.290{\scriptsize $\pm$0.001} & \textbf{0.543}{\scriptsize $\pm$0.005} & \textbf{3.511}{\scriptsize $\pm$0.024} \\
\bottomrule
\end{tabular}
\end{table*}

\begin{table*}[ht]
\centering
\caption{Full results on the \textbf{PartiPrompts} benchmark comparing Default, Dynamic Temperature, and $p$-less cluster decoding strategies. Metrics follow the same protocol as Table~\ref{tab:geneval}. Bold indicates the best value per model per metric.}
\label{tab:parti}
\begin{tabular}{p{1.6cm} p{3.5cm} r r r r r}
\toprule
\multirow{2}{*}{Model} & \multirow{2}{*}{Decoding Strategy} & \multirow{2}{*}{Quality} & \multicolumn{2}{c}{Alignment} & \multicolumn{2}{c}{Diversity} \\
\cmidrule(lr){4-5}\cmidrule(lr){6-7}
 & & VLM & VLM & CLIPScore & DINOv2 & Vendi \\
\midrule
\multirow{3}{*}{Anole} & Default & 4.154{\scriptsize $\pm$0.018} & 3.428{\scriptsize $\pm$0.035} & 0.296{\scriptsize $\pm$0.001} & 0.535{\scriptsize $\pm$0.005} & 3.401{\scriptsize $\pm$0.024} \\
 & Dynamic Temperature & \textbf{4.320}{\scriptsize $\pm$0.025} & \textbf{3.563}{\scriptsize $\pm$0.051} & \textbf{0.297}{\scriptsize $\pm$0.002} & 0.531{\scriptsize $\pm$0.008} & 3.353{\scriptsize $\pm$0.037} \\
 & $p$-less cluster (Ours) & 3.816{\scriptsize $\pm$0.024} & 3.289{\scriptsize $\pm$0.049} & 0.290{\scriptsize $\pm$0.002} & \textbf{0.542}{\scriptsize $\pm$0.007} & \textbf{3.466}{\scriptsize $\pm$0.030} \\
\midrule
\multirow{3}{*}{Emu3} & Default & 4.684{\scriptsize $\pm$0.020} & 3.798{\scriptsize $\pm$0.049} & 0.295{\scriptsize $\pm$0.002} & 0.364{\scriptsize $\pm$0.007} & 2.618{\scriptsize $\pm$0.032} \\
 & Dynamic Temperature & \textbf{4.704}{\scriptsize $\pm$0.017} & \textbf{3.862}{\scriptsize $\pm$0.047} & \textbf{0.295}{\scriptsize $\pm$0.002} & 0.360{\scriptsize $\pm$0.007} & 2.611{\scriptsize $\pm$0.031} \\
 & $p$-less cluster (Ours) & 4.233{\scriptsize $\pm$0.027} & 3.303{\scriptsize $\pm$0.047} & 0.287{\scriptsize $\pm$0.002} & \textbf{0.400}{\scriptsize $\pm$0.007} & \textbf{2.794}{\scriptsize $\pm$0.031} \\
\midrule
\multirow{3}{*}{Janus Pro-7B} & Default & 3.959{\scriptsize $\pm$0.041} & 3.403{\scriptsize $\pm$0.056} & 0.288{\scriptsize $\pm$0.002} & 0.426{\scriptsize $\pm$0.007} & 2.947{\scriptsize $\pm$0.031} \\
 & Dynamic Temperature & \textbf{4.503}{\scriptsize $\pm$0.024} & \textbf{4.007}{\scriptsize $\pm$0.045} & \textbf{0.302}{\scriptsize $\pm$0.002} & 0.375{\scriptsize $\pm$0.007} & 2.707{\scriptsize $\pm$0.033} \\
 & $p$-less cluster (Ours) & 3.724{\scriptsize $\pm$0.044} & 3.179{\scriptsize $\pm$0.059} & 0.281{\scriptsize $\pm$0.002} & \textbf{0.454}{\scriptsize $\pm$0.007} & \textbf{3.080}{\scriptsize $\pm$0.031} \\
\midrule
\multirow{3}{*}{LlamaGen} & Default & 3.849{\scriptsize $\pm$0.024} & 3.421{\scriptsize $\pm$0.048} & 0.291{\scriptsize $\pm$0.002} & 0.474{\scriptsize $\pm$0.006} & 3.192{\scriptsize $\pm$0.027} \\
 & Dynamic Temperature & \textbf{3.913}{\scriptsize $\pm$0.025} & \textbf{3.479}{\scriptsize $\pm$0.052} & \textbf{0.293}{\scriptsize $\pm$0.002} & 0.471{\scriptsize $\pm$0.006} & 3.164{\scriptsize $\pm$0.030} \\
 & $p$-less cluster (Ours) & 3.831{\scriptsize $\pm$0.024} & 3.418{\scriptsize $\pm$0.048} & 0.289{\scriptsize $\pm$0.001} & \textbf{0.479}{\scriptsize $\pm$0.006} & \textbf{3.213}{\scriptsize $\pm$0.026} \\
\bottomrule
\end{tabular}
\end{table*}

\begin{table*}[ht]
\centering
\caption{Full results on the \textbf{GenEval} benchmark for Janus Pro-7B at low CFG (CFG = 1.0)}
\label{tab:low_cfg_janus_pro_geneval}
\begin{tabular}{p{1.6cm} p{3.5cm} r r r r r r }
\toprule
\multirow{2}{*}{Model} & \multirow{2}{*}{Decoding Strategy} & \multirow{2}{*}{Quality} & \multicolumn{2}{c}{Alignment} & \multicolumn{2}{c}{Diversity} \\
\cmidrule(lr){4-5}\cmidrule(lr){6-7}
  & & VLM & VLM & CLIPScore & DINOv2 & Vendi \\
\midrule
\multirow{3}{*}{Janus Pro-7B} & Default & 4.417{\scriptsize $\pm$0.020} & 3.987{\scriptsize $\pm$0.035} & \textbf{0.319}{\scriptsize $\pm$0.001} & 0.442{\scriptsize $\pm$0.008} & 2.977{\scriptsize $\pm$0.034} \\
 & Dynamic Temperature & \textbf{4.503}{\scriptsize $\pm$0.018} & \textbf{4.115}{\scriptsize $\pm$0.034} & 0.319{\scriptsize $\pm$0.003} & 0.426{\scriptsize $\pm$0.016} & 2.912{\scriptsize $\pm$0.069} \\
 & $p$-less cluster (Ours) & 3.804{\scriptsize $\pm$0.023} & 3.174{\scriptsize $\pm$0.042} & 0.295{\scriptsize $\pm$0.001} & \textbf{0.552}{\scriptsize $\pm$0.006} & \textbf{3.506}{\scriptsize $\pm$0.028} \\
\bottomrule
\end{tabular}
\end{table*}

\begin{table*}[ht]
\centering
\caption{Full results on the \textbf{PartiPrompts} benchmark for Janus Pro-7B at low CFG (CFG = 1.0)}
\label{tab:low_cfg_janus_pro_parti}
\begin{tabular}{p{1.6cm} p{3.5cm} r r r r r r }
\toprule
\multirow{2}{*}{Model} & \multirow{2}{*}{Decoding Strategy} & \multirow{2}{*}{Quality} & \multicolumn{2}{c}{Alignment} & \multicolumn{2}{c}{Diversity} \\
\cmidrule(lr){4-5}\cmidrule(lr){6-7}
  & & VLM & VLM & CLIPScore & DINOv2 & Vendi \\
\midrule
\multirow{3}{*}{Janus Pro-7B} & Default  & 4.271{\scriptsize $\pm$0.027} & 3.358{\scriptsize $\pm$0.056} & 0.276{\scriptsize $\pm$0.002} & 0.553{\scriptsize $\pm$0.010} & 3.489{\scriptsize $\pm$0.045} \\
 & Dynamic Temperature & \textbf{4.379}{\scriptsize $\pm$0.035} & \textbf{3.376}{\scriptsize $\pm$0.074} & \textbf{0.295}{\scriptsize $\pm$0.003} & 0.476{\scriptsize $\pm$0.013} & 3.142{\scriptsize $\pm$0.059} \\
 & $p$-less cluster (Ours) & 4.053{\scriptsize $\pm$0.029} & 3.130{\scriptsize $\pm$0.056} & 0.273{\scriptsize $\pm$0.002} & \textbf{0.589}{\scriptsize $\pm$0.010} & \textbf{3.646}{\scriptsize $\pm$0.042} \\
\bottomrule
\end{tabular}
\end{table*}

\subsection{Diversity breakdownby PartiPrompts challenge category (all models)}
\label{sec:appendix_challenge_parti}

Figure~\ref{fig:challenge_parti_apx} extends Figure~\ref{fig:challenge_parti} in the main text by showing DINOv2 diversity broken down by PartiPrompts challenge category for all four AR models.

\begin{figure*}[ht]
  \centering
  \begin{subfigure}[t]{0.48\linewidth}
    \includegraphics[width=\linewidth]{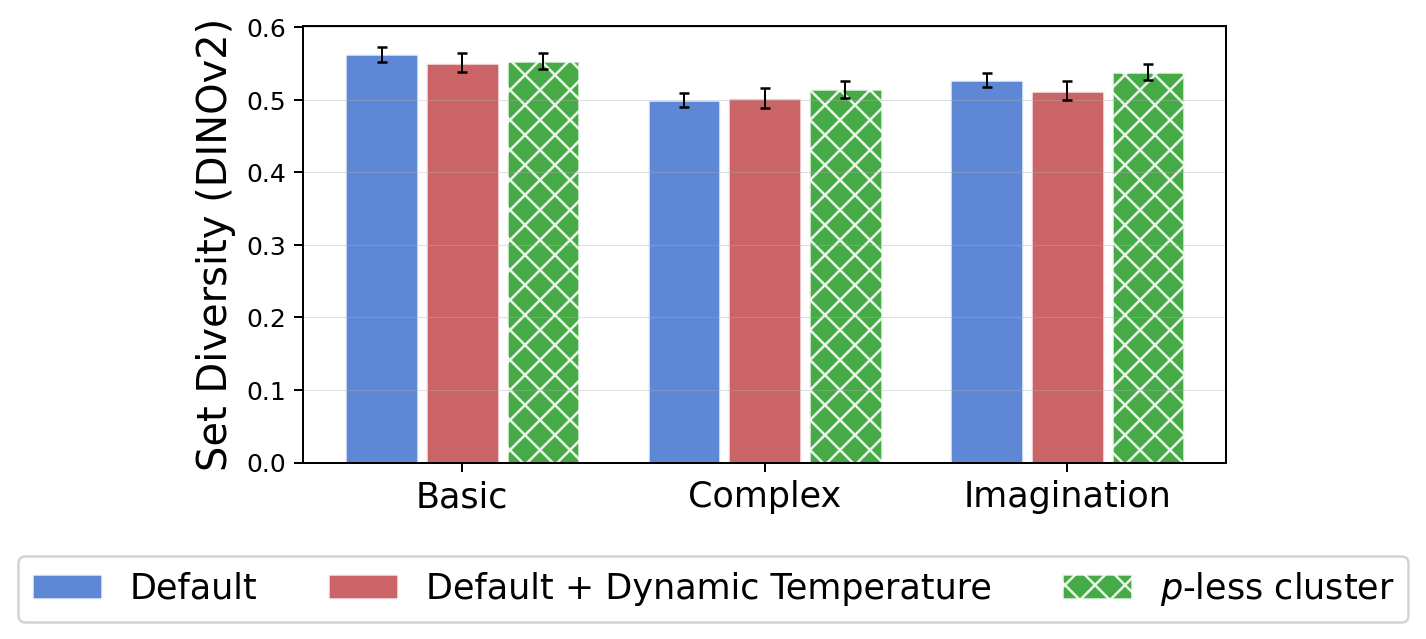}
    \caption{Anole}
  \end{subfigure}
  \hfill
  \begin{subfigure}[t]{0.48\linewidth}
    \includegraphics[width=\linewidth]{figures/paper_table4_challenge_full_parti_emu3_dino.png}
    \caption{Emu3}
  \end{subfigure}
  \vspace{0.5em}
  \begin{subfigure}[t]{0.48\linewidth}
    \includegraphics[width=\linewidth]{figures/paper_table4_challenge_full_parti_janus_dino.png}
    \caption{Janus Pro}
  \end{subfigure}
  \hfill
  \begin{subfigure}[t]{0.48\linewidth}
    \includegraphics[width=\linewidth]{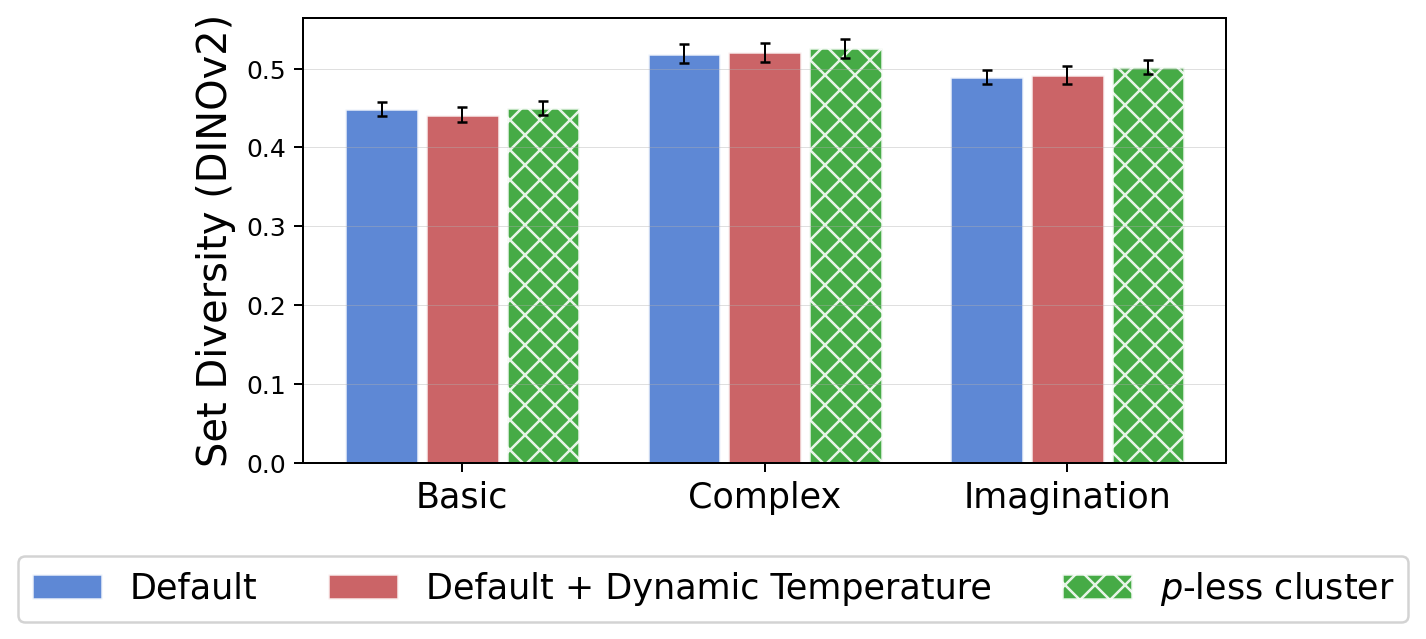}
    \caption{LlamaGen}
  \end{subfigure}
  \caption{\textbf{DINOv2 diversity by PartiPrompts challenge category across all four AR models.} $p$-less cluster consistently achieves the highest diversity in the \emph{Complex} and \emph{Imagination} categories across all models. For Anole and LlamaGen, the gains extend uniformly across all categories, whereas for Emu3 and Janus Pro the improvements are more concentrated in open-ended prompts.}
  \label{fig:challenge_parti_apx}
\end{figure*}

\subsection{Additional diversity-quality tradeoff plots}
\label{sec:appendix_tradeoff_plots}

\begin{figure}[t]
  \centering
  \includegraphics[width=\linewidth]{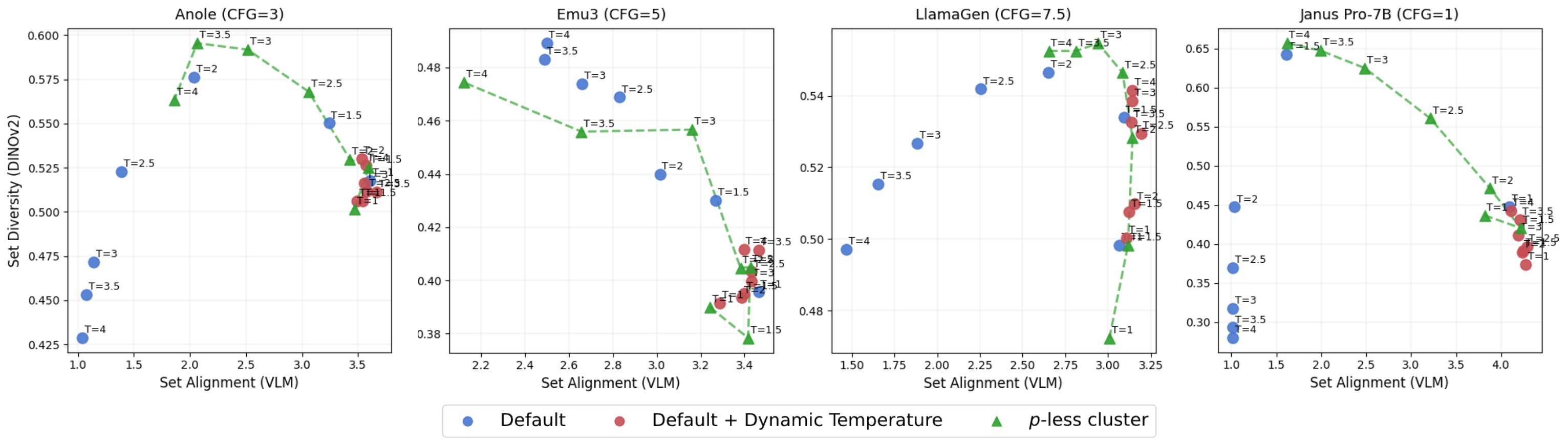}
  \caption{\textbf{Diversity--alignment tradeoff across decoding strategies for all 4 models using GenEval Dataset (DINOv2 vs.\ Prompt alignment) .} Janus Pro-7B uses low CFG (CFG$=$1.0) due to its low diversity under default CFG (Section \ref{sec:main_results}). Our $p$-less cluster method (triangles) shifts the Pareto frontier toward higher diversity while remaining competitive on alignment (except for Emu3).}
  \label{fig:apx_diversity_quality_tradeoff}
\end{figure}

\clearpage
\section{VLM-as-judge prompt and protocol}
\label{sec:appendix_vlm_judge}

The following prompt was used for all VLM-as-judge evaluations. The model is instructed to act as an expert image quality evaluator, scoring each generated image on both visual quality and prompt alignment. The protocol ensures consistency and transparency in the evaluation process.

\begin{tcolorbox}[
  title={Listing A.1: VLM-as-Judge Evaluation Prompt},
  fonttitle=\small\bfseries,
  colback=white,
  colframe=black!50,
  colbacktitle=black!75,
  coltitle=white,
  boxrule=0.4pt,
  enhanced,
  before upper={\ttfamily\small\setlength{\parskip}{0.2em}}
]
You are an expert image quality evaluator. You will assess a generated image against its source prompt with precision and consistency.

\textbf{Source Prompt}

"\{prompt\_text\}"

\textbf{Evaluation criteria}

\textbf{1. Visual quality (1--5)}
Assess the intrinsic quality of the image, independent of the prompt:
\begin{itemize}\setlength{\itemsep}{0pt}
  \item 5 -- Excellent: Sharp, well-lit, natural colors, zero artifacts, professional-grade composition
  \item 4 -- Good: Minor imperfections (slight blur, minor noise) that don't detract from the image
  \item 3 -- Acceptable: Noticeable issues (soft focus, unnatural lighting, minor distortions) but still coherent
  \item 2 -- Poor: Significant artifacts, distortions, or incoherence that impairs the image
  \item 1 -- Very Poor: Severely corrupted, unrecognizable, or fundamentally broken
\end{itemize}

\textbf{2. Prompt alignment (1--5)}
Assess how faithfully the image represents the prompt. Consider: subject, attributes (color, size, texture), spatial relationships, style, and mood:
\begin{itemize}\setlength{\itemsep}{0pt}
  \item 5 -- Perfect: All prompt elements present and accurately depicted
  \item 4 -- Good: Core elements present; minor attributes missing or slightly off
  \item 3 -- Partial: Main subject recognizable but several attributes or details are wrong/missing
  \item 2 -- Poor: Only loosely related to the prompt; major elements absent or incorrect
  \item 1 -- No match: Image bears no meaningful relation to the prompt
\end{itemize}

\textbf{Scoring rules}
\begin{itemize}\setlength{\itemsep}{0pt}
  \item Score each dimension \textbf{independently} --- a technically beautiful image can still score 1 on alignment
  \item Be \textbf{consistent}: the same image shown twice should receive the same scores
  \item Avoid \textbf{leniency bias}: reserve 5s for images that genuinely meet all criteria
  \item If the prompt is ambiguous, judge alignment against the most reasonable interpretation
\end{itemize}

\textbf{Output format}
Return ONLY a valid JSON object. No preamble, no trailing text.

\{

\quad "quality\_score": <int 1-5>,

\quad "alignment\_score": <int 1-5>,

\quad "quality\_reasoning": "<one sentence citing specific visual evidence>",

\quad "alignment\_reasoning": "<one sentence citing prompt elements present or missing>"

\}
\end{tcolorbox}

\section{Qualitative examples}
\label{sec:appendix_qualitative}

Figures~\ref{fig:qualitative_emu3}--\ref{fig:appendix_llamagen} show qualitative comparisons for all 4 models on GenEval and PartiPrompts.

\begin{figure*}[ht]
    \centering
    \includegraphics[width=\linewidth]{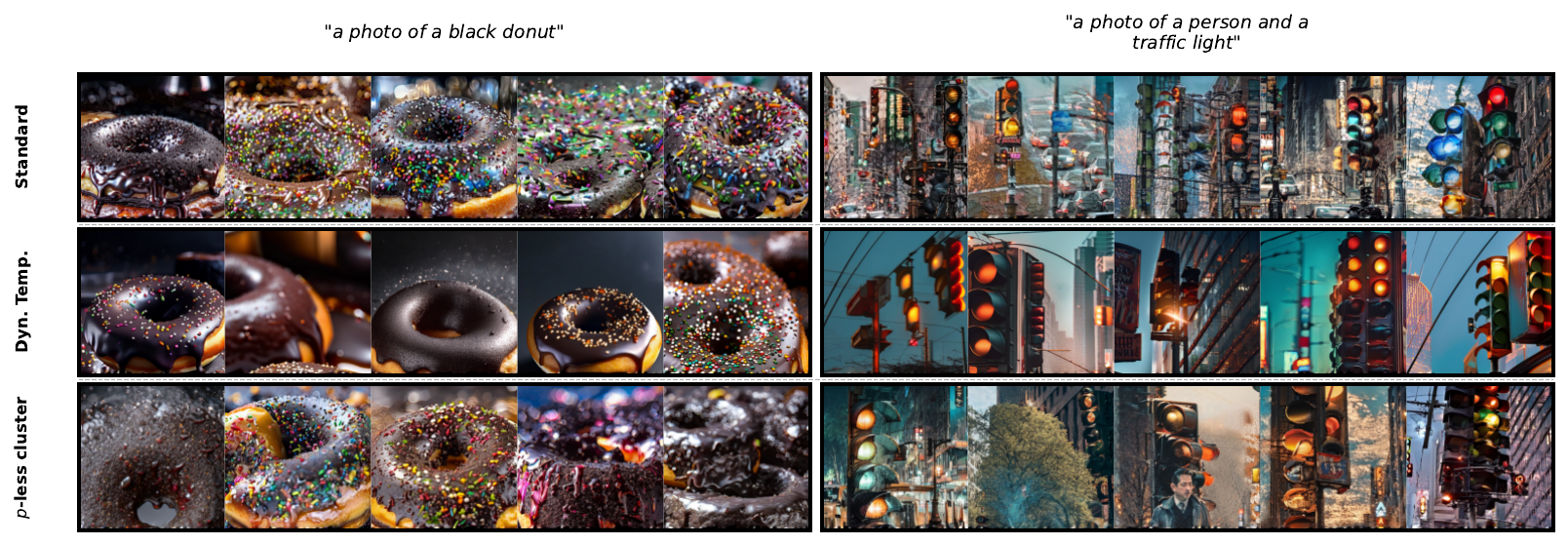}
    \vspace{0.2em}
    \includegraphics[width=\linewidth]{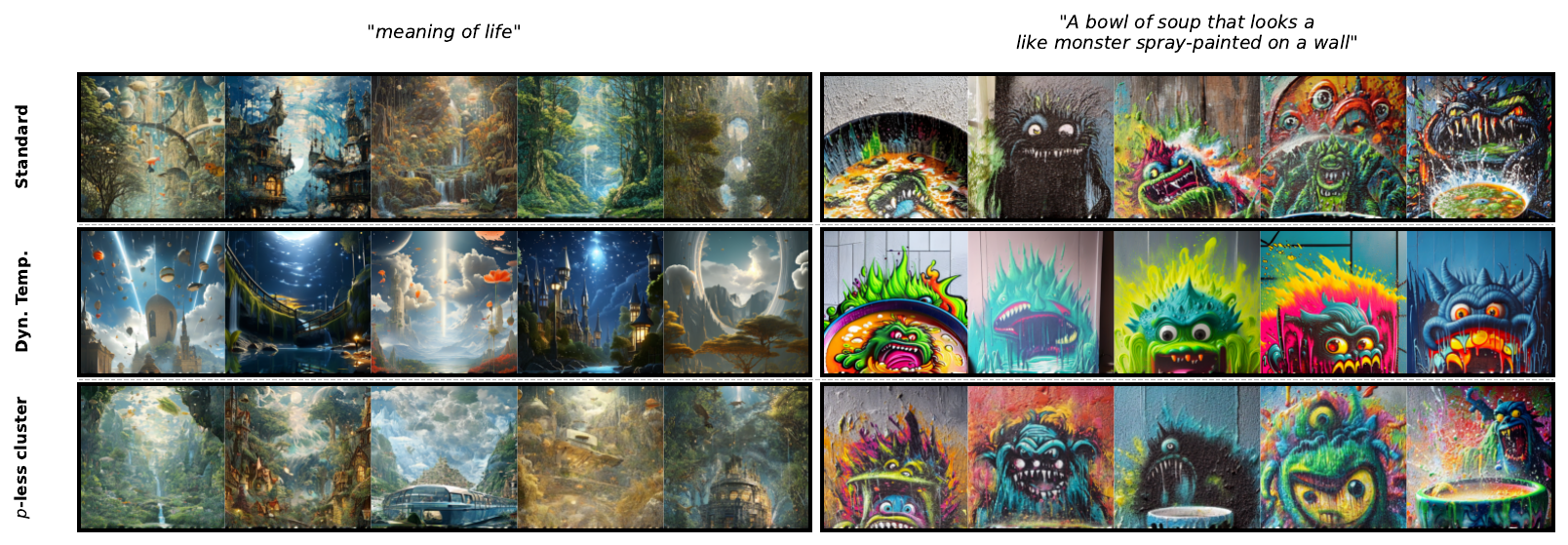}
    \caption{\textbf{Qualitative comparison for Emu3.} Hyperparameters match Table~\ref{tab:main_results}. Top: GenEval prompts; Bottom: PartiPrompts. Default and Dynamic Temperature produce visually similar samples for the same prompt, reflecting diversity collapse. $p$-less cluster method generates samples with greater variation in object position, color, and scene composition while preserving prompt fidelity. Additional qualitative examples for all models are provided in Appendix~\ref{sec:appendix_qualitative}.}
    \label{fig:qualitative_emu3}
\end{figure*}

\begin{figure}[ht]
    \centering
    \includegraphics[width=\linewidth]{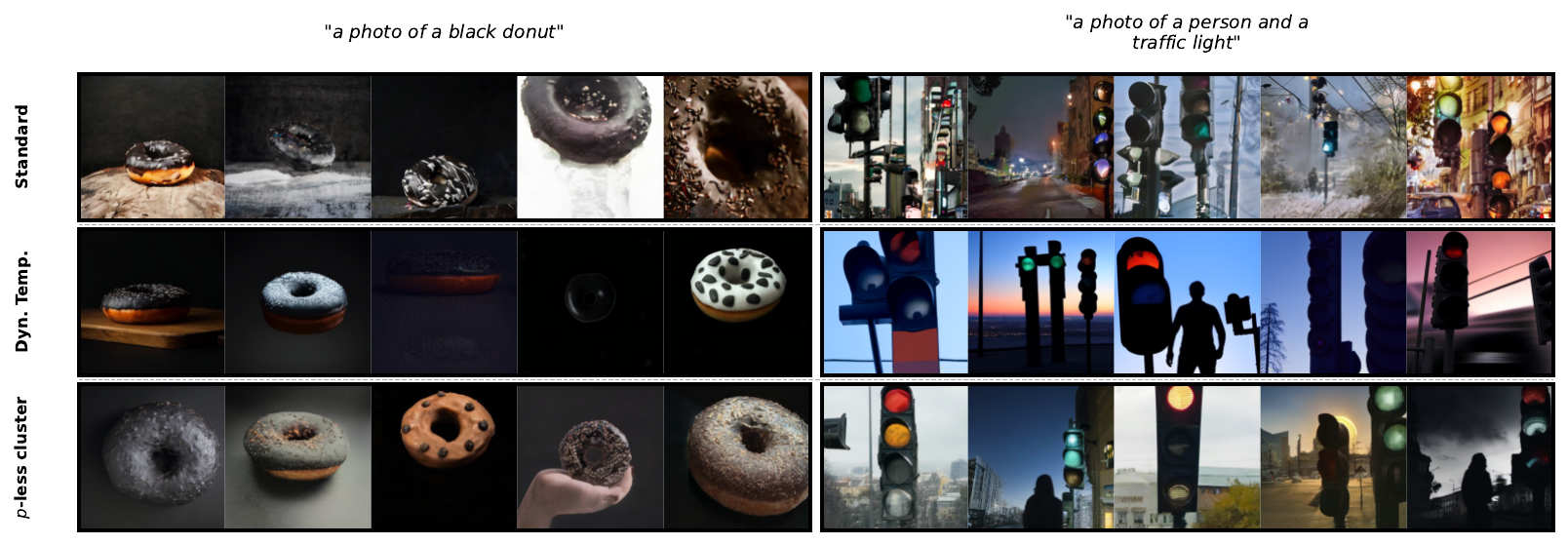}
    \vspace{0.5em}
    \includegraphics[width=\linewidth]{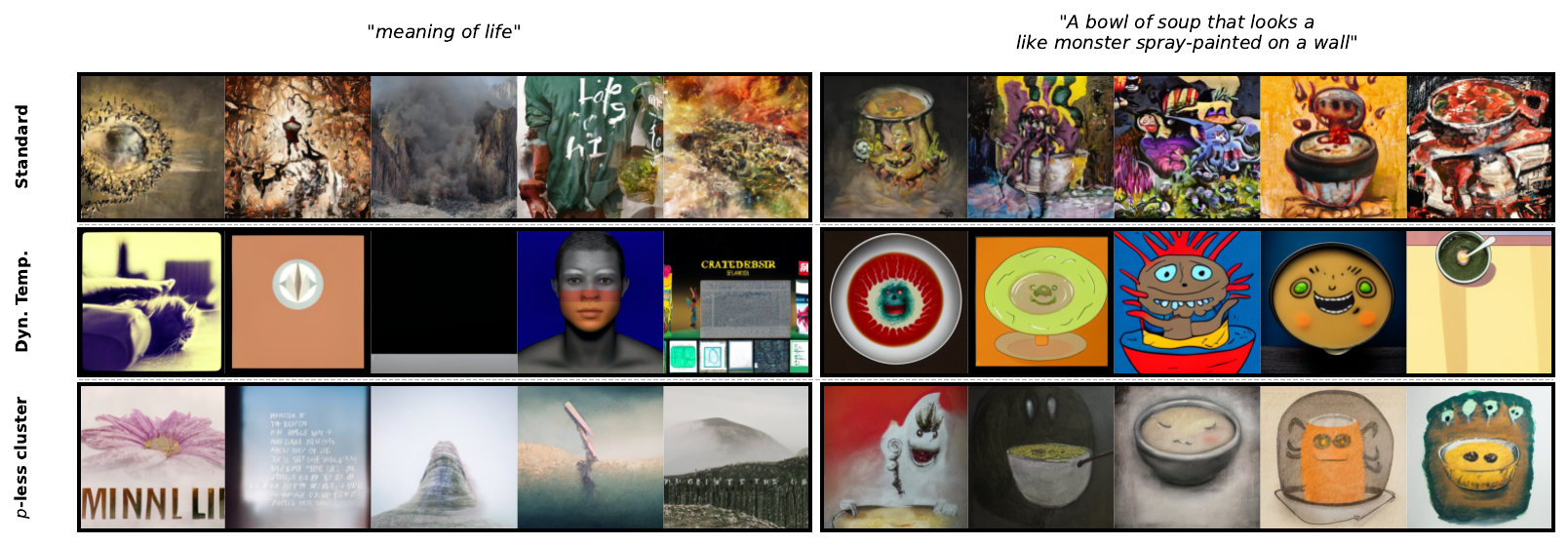}
    \caption{\textbf{Qualitative comparison for Anole.} Top: GenEval prompts; Bottom: PartiPrompts. Hyperparameters match Table~\ref{tab:geneval}.}
    \label{fig:appendix_anole}
\end{figure}

\begin{figure}[ht]
    \centering
    \includegraphics[width=\linewidth]{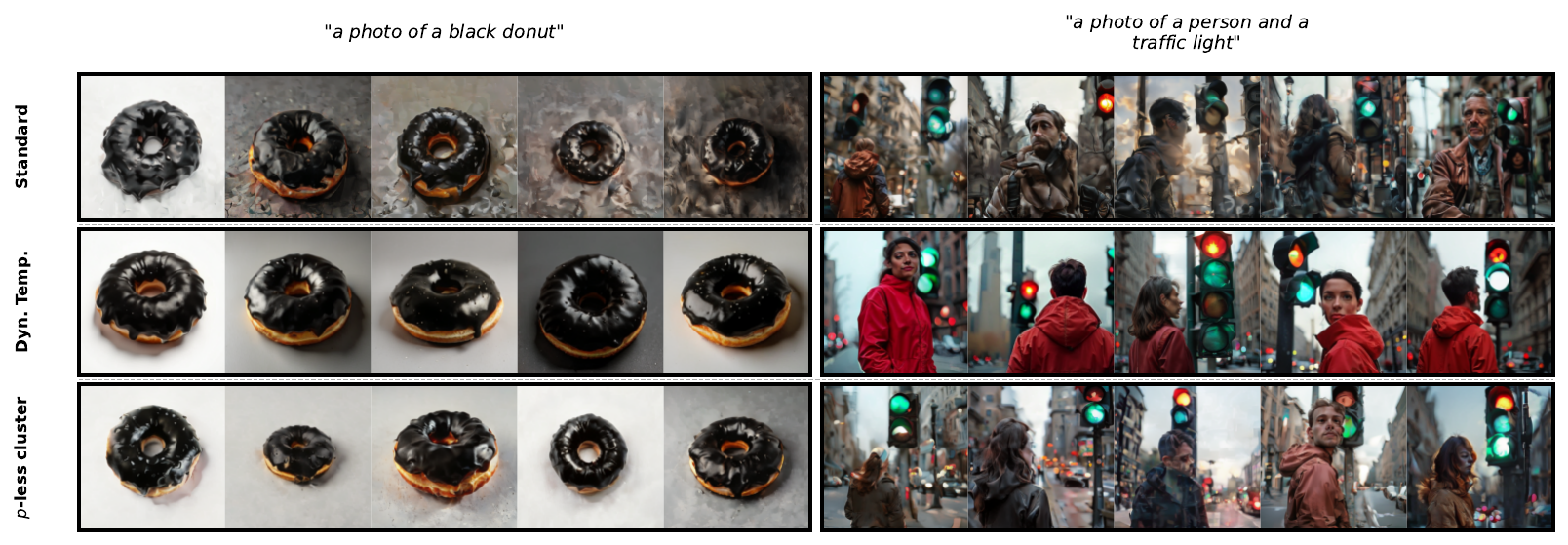}
    \vspace{0.5em}
    \includegraphics[width=\linewidth]{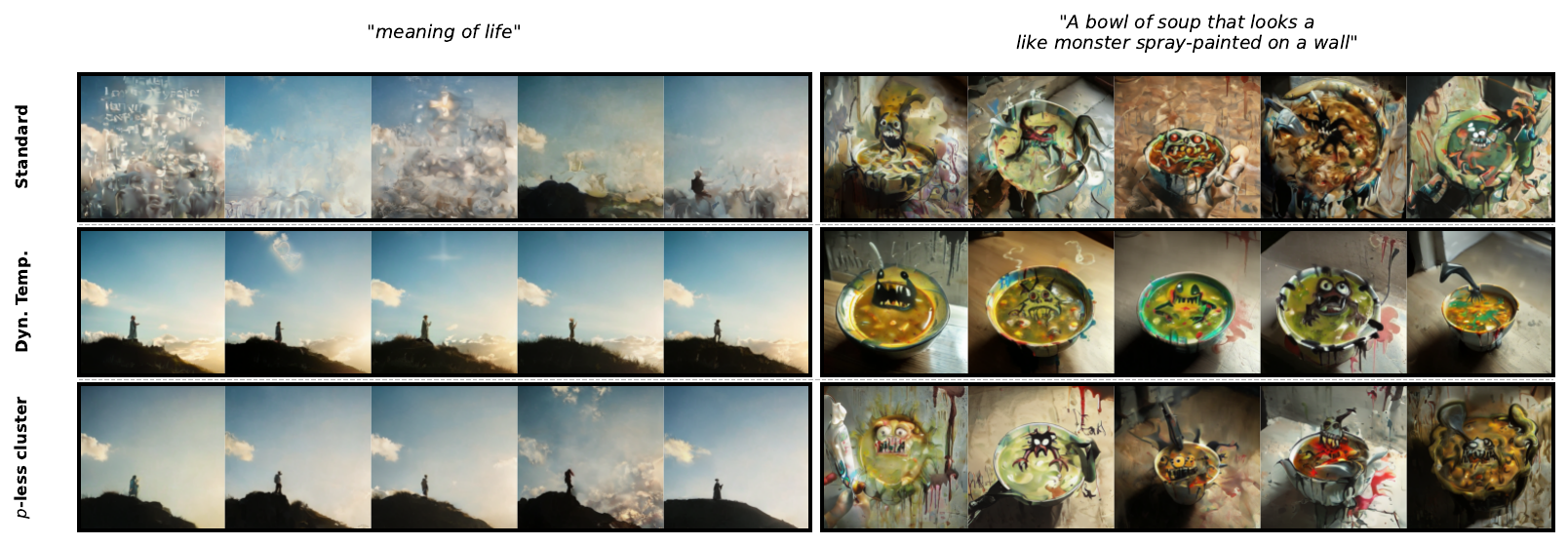}
    \caption{\textbf{Qualitative comparison for Janus} at default CFG. Top: GenEval prompts; Bottom: PartiPrompts. Hyperparameters match Table~\ref{tab:geneval}.}
    \label{fig:appendix_janus}
\end{figure}

\begin{figure}[ht]
    \centering
    \includegraphics[width=\linewidth]{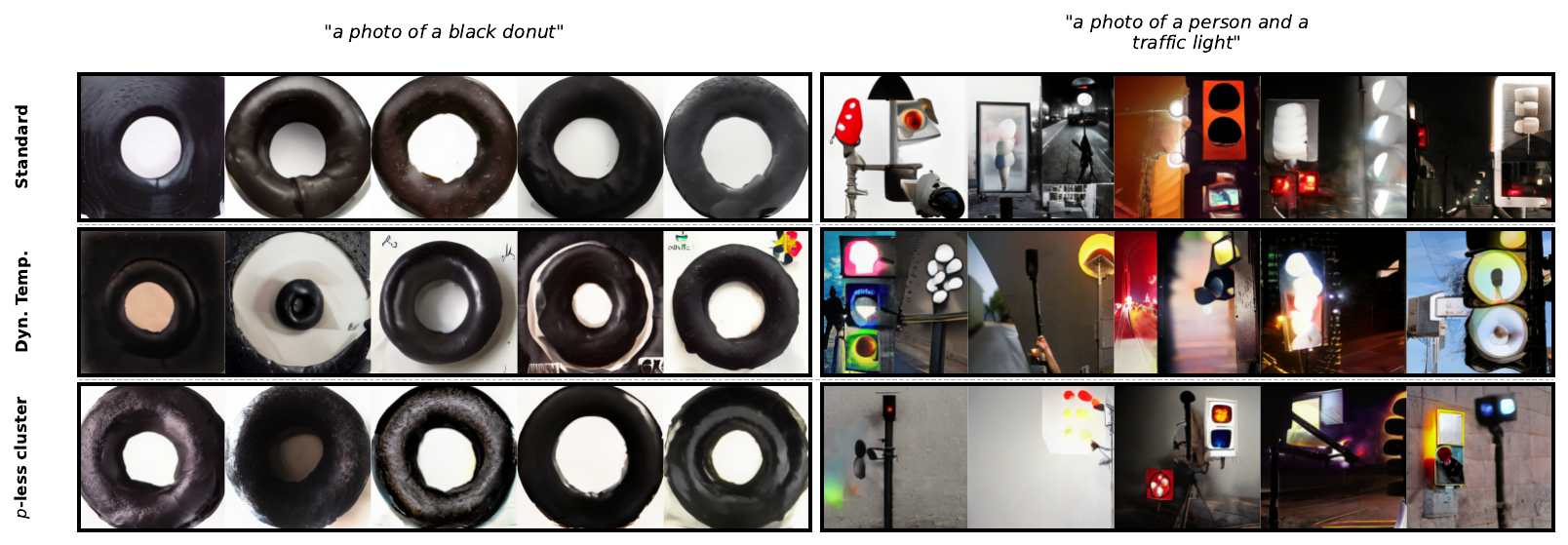}
    \vspace{0.5em}
    \includegraphics[width=\linewidth]{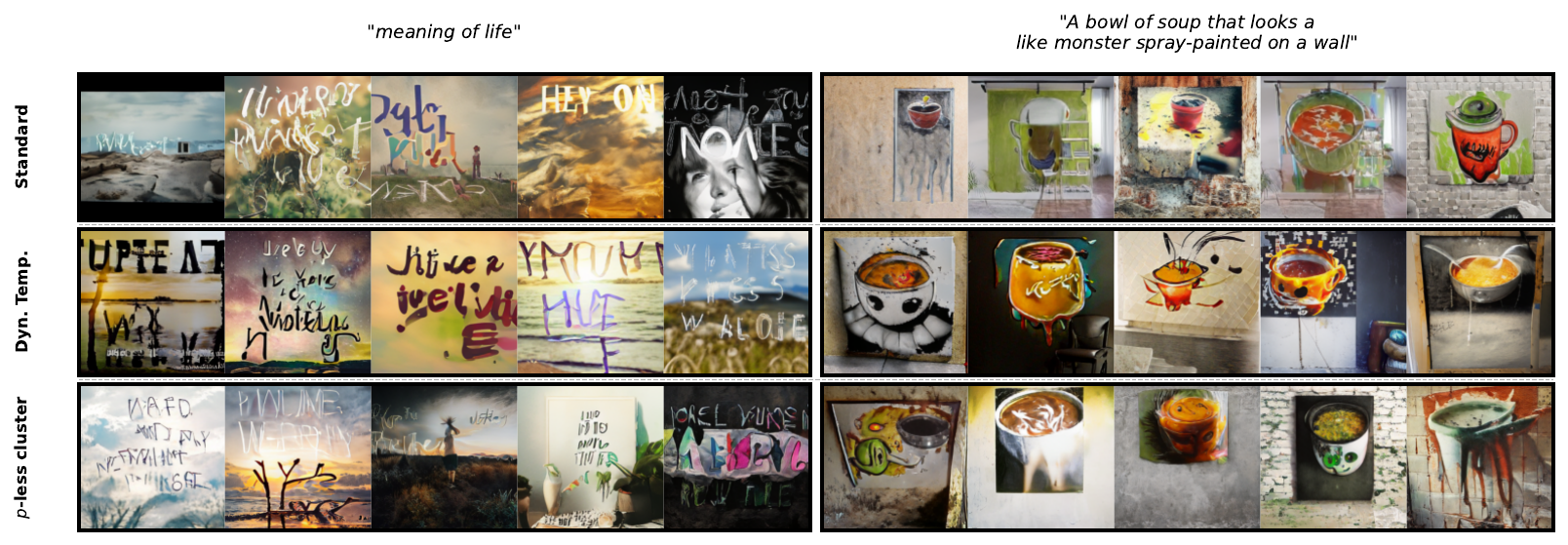}
    \caption{\textbf{Qualitative comparison for LLaMAGen.} Top: GenEval prompts; Bottom: PartiPrompts. Hyperparameters match Table~\ref{tab:geneval}.}
    \label{fig:appendix_llamagen}
\end{figure}

\clearpage

\section{Additional ablation studies}

Table~\ref{tab:p_less_comparison} compares $p$-less and $p$-less cluster at the default CFG value across all four models and both benchmarks. $p$-less cluster improves Vendi diversity in the large majority of settings, with the most substantial gains on Janus Pro-7B (+16.54\% on GenEval, +10.12\% on PartiPrompts), while gains on Anole and LlamaGen, the two under-trained models, are modest and occasionally slightly negative.

\begin{table*}[h]
\centering
\caption{Vendi diversity score comparison of $p$-less vs.\ $p$-less cluster across models and benchmarks at default CFG. $p$-less cluster yields diversity gain compared to $p$-less in most model-dataset settings with the most significant gains in Janus Pro-7B. }
\label{tab:p_less_comparison}
\begin{tabular}{ll rrr}
\toprule
Model & Dataset & $p$-less & $p$-less cluster & \% Gain / Loss \\
\midrule
\multirow{2}{*}{Anole}
  & GenEval & 3.578 & 3.549 & \textcolor{red}{$-$0.81\%} \\
  & PartiPrompt   & 3.443 & 3.466 & \textcolor{teal}{$+$0.67\%} \\
\midrule
\multirow{2}{*}{Emu3}
  & GenEval & 2.884 & 2.973 & \textcolor{teal}{$+$3.09\%} \\
  & PartiPrompt   & 2.779 & 2.794 & \textcolor{teal}{$+$0.54\%} \\
\midrule
\multirow{2}{*}{Janus Pro-7B}
  & GenEval & 2.128 & 2.480 & \textcolor{teal}{$+$16.54\%} \\
  & PartiPrompt   & 2.797 & 3.080 & \textcolor{teal}{$+$10.12\%} \\
\midrule
\multirow{2}{*}{LlamaGen}
  & GenEval & 3.441 & 3.511 & \textcolor{teal}{$+$2.03\%} \\
  & PartiPrompt   & 3.230 & 3.216 & \textcolor{red}{$-$0.43\%} \\
\bottomrule
\end{tabular}
\end{table*}

\section{Theoretical justification}
\newtcbtheorem[number within=section]{myprop}{Proposition}%
{
    enhanced,
    breakable,
    colback=white,
    colframe=blue!70!black,
    coltitle=white,
    fonttitle=\bfseries,
    colbacktitle=blue!70!black,
    attach boxed title to top left=
    {xshift=0.5em,yshift=-2mm},
    boxed title style=
    {
        sharp corners,
        rounded corners=northwest,
        rounded corners=northeast,
        boxrule=0pt,
    },
    sharp corners,
    boxrule=0.8pt,
    left=8pt,
    right=8pt,
    top=4pt,
    bottom=4pt,
}{prop}

Token-level sampling treats each token independently, even though autoregressive image models often spread probability mass across many visually similar tokens within a cluster. As a result, token-level filtering can remove tokens that collectively represent a meaningful generation mode, reducing diversity. Cluster-level $p$-less instead operates on aggregate cluster probabilities, preserving all tokens from clusters that survive truncation. The tokens most affected by this difference are those filtered at the token level but potentially retained at the cluster level, which we term \textit{High-Support Low-Probability} (HSLP) tokens (\eqref{eq:hslp1}, \eqref{eq:hslp2}). We show empirically that HSLP tokens are common in practice and prove that cluster-level $p$-less assigns them strictly higher probability than token-level $p$-less.
\begin{figure*}[h]
    \centering
    \includegraphics[width=0.65\linewidth]{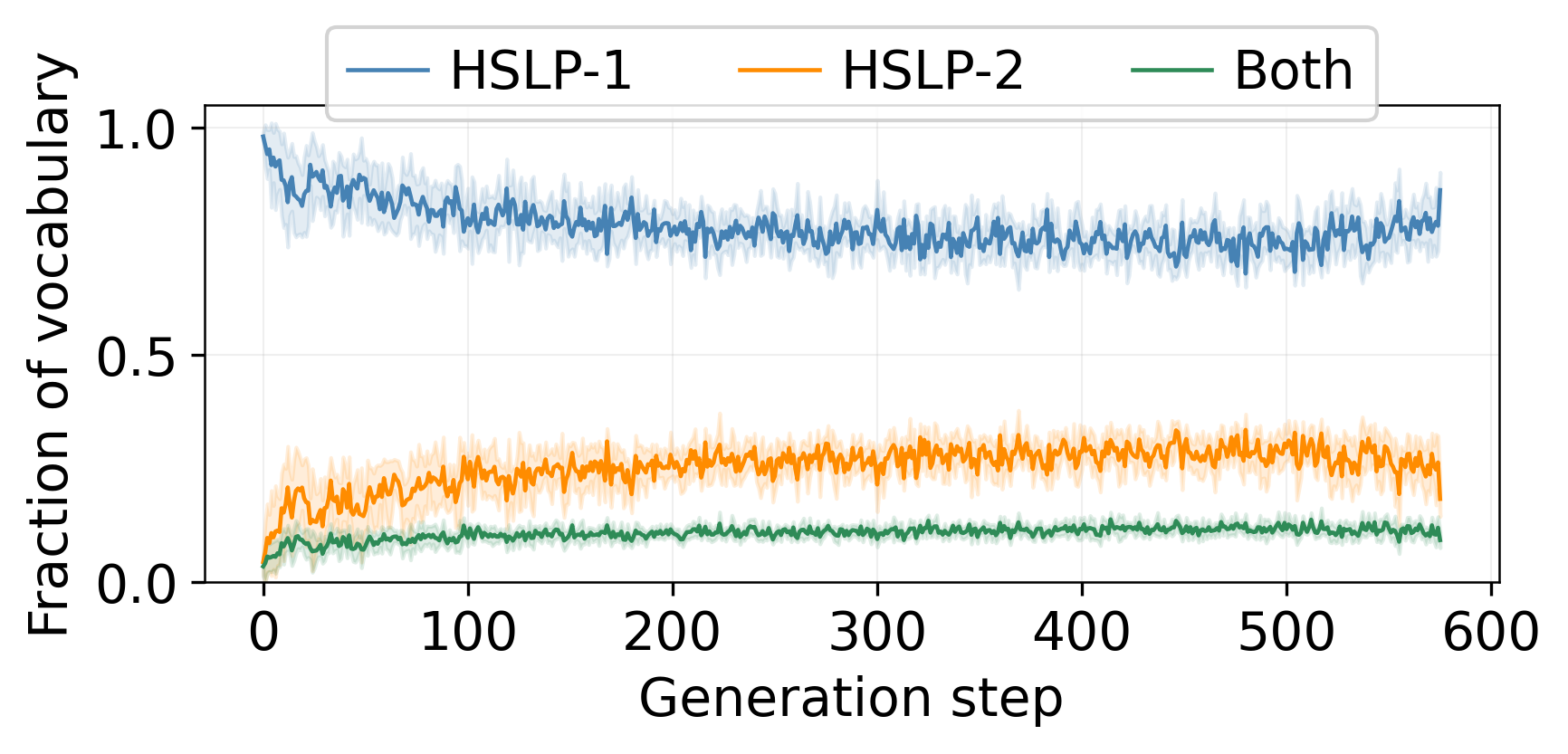}
    \caption{\textbf{Fraction of tokens satisfying the HSLP conditions (\eqref{eq:hslp1}, \eqref{eq:hslp2}) for Janus Pro-7B across decoding steps.} A reasonable portion of tokens (approximately 10\%) satisfy both conditions, indicating that they are filtered by token-level $p$-less but would survive cluster-level $p$-less, thus receiving a non-zero probability under our method.}
    \label{fig:janus_hslp_ratio}
\end{figure*}
\subsection{Empirical prevalence of HSLP tokens}

Figure~\ref{fig:janus_hslp_ratio} shows the fraction of tokens satisfying the HSLP conditions for Janus Pro-7B across decoding steps. Approximately 10\% of tokens at each step are filtered out by token-level $p$-less yet belong to a cluster that survives cluster-level filtering. This motivates the formal analysis in Proposition H.1 below: such tokens receive strictly higher probability under cluster-level $p$-less, and recovering them contributes to improved diversity without sacrificing quality.

\subsection{Theoretical analysis of token admission}

\begin{myprop}{HSLP Token Admission}{hslp}

\noindent\textbf{Assumptions and Notation.}
Let $V$ denote the token vocabulary, partitioned into $N$ clusters $\{C_1, C_2, \ldots, C_N\}$. Let $P_{\text{model}}$ denote the autoregressive model distribution over tokens, and define the induced cluster distribution as
\[
P_{\text{cluster}}(C_n) = \sum_{x_i \in C_n} P_{\text{model}}(x_i).
\]
Let $x_c \in C_n$ be a token such that $P_{\text{model}}(x_c) > 0$ and satisfying the following \textit{High-Support Low-Probability} (HSLP) conditions:
\begin{align}
    P_{\text{model}}(x_c) &< L[P_{\text{model}}]
    = \sum_{i=1}^{|V|} P_{\text{model}}(x_i)^2
    \tag{HSLP-1}
    \label{eq:hslp1} \\
    P_{\text{cluster}}(C_n) &\geq
    L[P_\text{cluster}] =
    \sum_{i=1}^{N} P_\text{cluster}(C_i)^2
    \tag{HSLP-2}
    \label{eq:hslp2}
\end{align}

Define the surviving cluster set as
\begin{equation}
    \mathcal{C}_{p\text{-less-cluster}}
    =
    \left\{
    C_k :
    P_{\text{cluster}}(C_k)
    \geq
    L[P_\text{cluster}]
    \right\},
\end{equation}
where
$L[P_\text{cluster}]$
denotes the cluster-level collision probability.
\\ 

Let $\tilde{P}$ denote normalization of a non-negative measure to sum to unity.

\textbf{Claim.}
Under the two HSLP conditions,
\begin{equation}
    \tilde{P}_{p\text{-less}}(x_c)
    <
    P_{\text{model}}(x_c)
    <
    \tilde{P}_{p\text{-less-cluster}}(x_c).
\end{equation}
\end{myprop}

\begin{proof}

\textbf{Part 1: $\tilde{P}_{p\text{-less}}(x_c)=0$.}
By \eqref{eq:hslp1}, $x_c$ falls strictly below the token-level Rényi-2 threshold and is filtered out by token-level $p$-less. After normalization,
\begin{equation}
    \tilde{P}_{p\text{-less}}(x_c)
    =
    0
    <
    P_{\text{model}}(x_c),
    \label{eq:part1}
\end{equation}
where the strict inequality follows from the assumption
$P_{\text{model}}(x_c) > 0$.

\paragraph{Part 2: $\tilde{P}_{p\text{-less-cluster}}(x_c) > P_{\text{model}}(x_c)$.}
Under cluster-level sampling,
\begin{equation}
    \tilde{P}_{p\text{-less-cluster}}(x_c)
    =
    \tilde{P}_{p\text{-less}}(C_n)
    \cdot
    P(x_c \mid C_n),
    \label{eq:twostage}
\end{equation}
where $\tilde{P}_{p\text{-less}}(C_n)$ denotes the normalized cluster probability after filtering, and $P(x_c \mid C_n)$ is the within-cluster conditional probability of $x_c$.

Expanding \eqref{eq:twostage},
\begin{equation}
    \tilde{P}_{p\text{-less}}(C_n)
    =
    \frac{
    P_{\text{cluster}}(C_n)
    }{
    \displaystyle
    \sum_{C_k \in \mathcal{C}_{p\text{-less-cluster}}}
    P_{\text{cluster}}(C_k)
    },
    \qquad
    P(x_c \mid C_n)
    =
    \frac{
    P_{\text{model}}(x_c)
    }{
    P_{\text{cluster}}(C_n)
    }.
    \label{eq:expand}
\end{equation}

Substituting \eqref{eq:expand} into \eqref{eq:twostage} gives
\begin{equation}
    \tilde{P}_{p\text{-less-cluster}}(x_c)
    =
    \frac{
    P_{\text{model}}(x_c)
    }{
    \displaystyle
    \sum_{C_k \in \mathcal{C}_{p\text{-less-cluster}}}
    P_{\text{cluster}}(C_k)
    }.
    \label{eq:cancel}
\end{equation}

Define
\begin{equation}
    0 < Z
    =
    \sum_{C_k \in \mathcal{C}_{p\text{-less-cluster}}}
    P_{\text{cluster}}(C_k)
    <
    1,
    \label{eq:Z}
\end{equation}
where the strict inequality follows because at least one cluster is filtered out under non-uniform cluster distributions (shown in~\cite{pless}).

Therefore,
\begin{equation}
    \tilde{P}_{p\text{-less-cluster}}(x_c)
    =
    \frac{
    P_{\text{model}}(x_c)
    }{
    Z
    }
    >
    P_{\text{model}}(x_c).
    \label{eq:part2}
\end{equation}

\paragraph{Conclusion.}
Combining \eqref{eq:part1} and \eqref{eq:part2},
\begin{equation}
    \tilde{P}_{p\text{-less}}(x_c)
    =
    0
    <
    P_{\text{model}}(x_c)
    <
    \tilde{P}_{p\text{-less-cluster}}(x_c)
    =
    \frac{
    P_{\text{model}}(x_c)
    }{
    Z
    }.
\end{equation}

\end{proof}

\begin{remark}
In the degenerate case where all clusters survive, i.e. under uniform cluster distribution, $Z=1$ and cluster-level $p$-less reduces to multinomial sampling.
\end{remark}

\clearpage

\end{document}